\documentclass[11pt]{article} % use larger type; default would be 10pt

\usepackage{macros}

\begin{document}

\begin{center}

	{\bf{\LARGE{Adversarial risk via optimal transport\\ \vspace{2mm} and optimal couplings}}}

	\vspace*{.25in}

 	\begin{tabular}{ccc}
 		{\large{Muni Sreenivas Pydi}} & \hspace*{.5in} & {\large{Varun Jog}}\\
 		{\large{\texttt{pydi@wisc.edu}}} & \hspace*{.5in} & {\large{\texttt{vjog@wisc.edu}}} 
 			\end{tabular}
 \begin{center}
 Department of Electrical \& Computer Engineering\\
 University of Wisconsin-Madison
 \end{center}

	\vspace*{.2in}

December 2019

	\vspace*{.2in}

\end{center}

%\author[ \hspace{-1ex}]{Varun Log \thanks{Department of Electrical and Computer Engineering, UW - Madison. Email: vjog@wisc.edu}}
%\affil[ ]{}

\begin{abstract}
 
Modern machine learning algorithms perform poorly on adversarially manipulated data. Adversarial risk quantifies the error of classifiers in adversarial settings; adversarial classifiers minimize adversarial risk. In this paper, we analyze adversarial risk and adversarial classifiers from an optimal transport perspective. We show that the optimal adversarial risk for binary classification with 0-1 loss is determined by an optimal transport cost between the probability distributions of the two classes. We develop optimal transport plans (probabilistic couplings) for univariate distributions such as the normal, the uniform, and the triangular distribution. We also derive optimal adversarial classifiers in these settings. Our analysis leads to algorithm-independent fundamental limits on adversarial risk, which we calculate for several real-world datasets. We extend our results to general loss functions under convexity and smoothness assumptions.
\footnote{This paper was presented in part at the International Conference on Machine Learning, July 2020, Virtual.}

%The accuracy of modern machine learning algorithms deteriorates severely on adversarially manipulated test data. 
%\emph{Optimal adversarial risk} quantifies the best error rate of any classifier in the presence of adversaries, and \emph{optimal adversarial classifiers} are sought that minimize adversarial risk. 
%In this paper, we investigate the optimal adversarial risk and optimal adversarial classifiers from an optimal transport perspective. 
%We present a new and simple approach to show that the optimal adversarial risk for binary classification with $0$-$1$ loss function is completely characterized by an optimal transport cost between the probability distributions of the two classes, for a suitably defined cost function. We propose a novel coupling strategy that achieves the optimal transport cost for several univariate distributions like Gaussian, uniform and triangular. Using the optimal couplings, we obtain the optimal adversarial classifiers in these settings and show how they differ from optimal classifiers in the absence of adversaries. Based on our analysis, we evaluate algorithm-independent fundamental limits on adversarial risk for CIFAR-10, MNIST, Fashion-MNIST and SVHN datasets, and Gaussian mixtures based on them. In addition to the $0$-$1$ loss, we also derive bounds on the deviation of optimal risk and optimal classifier in the presence of adversaries for continuous loss functions, that are based on the convexity and smoothness of the loss functions.
\end{abstract}

\section{Introduction}

Deep learning has had tremendous success in recent times, producing state-of-the-art results in image classification \citep{KriEtal12, HeEtal16}, game playing \citep{SilEtal16, SilEtal18, VinEtal19}, speech \citep{HinEtal12, GraEtal13} and natural language processing \citep{YouEtal18, DevEtal19}. However, Szegedy et al.~\citep{SzeEtal13} discovered that these algorithms are surprisingly vulnerable to minute adversarial perturbations. Many {\em adversarial attacks}~\citep{AthCar18, CarWag17, GooEtal14} and defenses~\citep{MadEtal18, PapEtal16, CisEtal17} have been proposed since. Often, the defenses are subsequently broken or are computationally intractable in practice. %Recent work has focused on certifiable defenses that are provably robust against a pre-specified class of adversaries \citep{CohEtal19, SinEtal17, RagEtal18}.

The reason for existence of adversarial examples in deep learning is unknown, but many explanations have been suggested. One line of work hypothesizes that adversarial examples are inevitable in certain high-dimensional settings \citep{ShaEtal19, MahEtal19}. %Specifically, it has been shown that any classifier that works on data  sampled from concentrated metric probability spaces is susceptible to a high adversarial risk. For instance, when the input distribution is uniform over a high dimensional sphere \citep{GilEtal18} or Boolean hypercube \citep{DioEtal18}, or when the latent space of the data is a high dimensional Gaussian \citep{FawEtal18}, the adversarial risk can be significantly higher than standard risk, even for small $\epsilon$. 
Goodfellow et al.~\citep{GooEtal14} propose that the reason for adversarial examples may be the linear nature of deep neural networks.
Ilyas et al.~\citep{IlyEtal19} propose that adversarial examples correspond to non-robust features in the data that are highly predictive, but brittle.
Moreover, it was recently proposed that adversarial risk may be fundamentally at odds with standard risk---a claim that finds support both in theory \citep{TsiEtal18} and in practice \citep{SuEtal18}.

In this paper, we deviate from algorithm-dependent investigations of adversarial examples and ask two fundamental questions:

\begin{question}\label{q_1}
How much can the optimal adversarial risk differ from optimal standard risk? In the binary classification setting, we may equivalently ask: How much will the classification error increase due to an adversary---an increase that cannot be mitigated by \emph{any} algorithm?
\end{question}

\begin{question}\label{q_2}
How does the optimal adversarial classifier differ from the standard optimal classifier?
\end{question}

Recent works have addressed Question~\ref{q_1} by deriving upper  and lower bounds on the optimal adversarial risk with respect to a fixed set of classifiers, by extending the PAC learning theory to encompass adversaries~\cite{KhiLoh18, YinEtal18}. 
A related question asks how much adversarial perturbation is sufficient to make the optimal adversarial risk significantly greater than the optimal standard risk. Relevant works in this direction develop robustness metrics that depend on the classifier \citep{WenEtal18, ZhaEtal18, HeiEtal17}. %In addition to Question~\ref{q_1}, one might also consider the nature of the optimal classifier under the standard and adversarial settings. This motivates the following question: 
For Question~\ref{q_2}, recall that the optimal classifier without an adversary is simply the Bayes optimal classifier. In adversarial settings, the optimal classifier may differ considerably from the Bayes optimal classifier. 
A recent line of work shows that optimal adversarial classifier can be calculated using non-parametric methods if data are well-separated \citep{YanEtal20, BhaCha20}.
The work of Moosavi-Dezfooli et al.~\citep{MooEtal18}, Cohen et al.~\citep{CohEtal19} and Yang et al.~\citep{YanEtal20_2} suggests that the optimal adversarial classifier has smoother boundaries than the optimal standard classifier. Even so, the question of how the optimal adversarial classifier differs from the standard one remains open. 

 %empirically observed that adversarial training significantly reduces the curvature of the loss function with respect to the input.  Another line of work attempts to construct a provably robust classifier from a baseline classifier using randomized smoothing \cite{CohEtal19}. These works provide a clue that the optimal classifier may in fact be different in the presence of an adversary. 
%{\red Move this to the discussion? [[}Even so, many other interesting questions remain. For instance, is the optimal classifier without an adversary approximately the same as the optimal classifier with a small adversary (i.e., small $\epsilon$)? If decision boundaries change, do they change smoothly with increasing strength of an adversary, or do they change drastically? {\red ]]}

The closest work to ours is Bhagoji et al.~\cite{BhaEtal19}, which develops algorithm-independent lower bounds for learning in the presence of an adversary. Specifically, \cite{BhaEtal19} contains a similar result to our Theorem~\ref{th_01bound} which gives the optimal adversarial risk for binary classification with $0$-$1$ loss in terms of an optimal transport cost between the probability distributions of the two classes. We provide a new, simpler proof of this characterization by applying the Kantorovich duality of optimal transport for $0$-$1$ cost functions. We shall discuss the results from Bhagoji et al.~\cite{BhaEtal19} and compare these with our results at appropriate points in the paper.

\subsection*{Our contributions}

In this paper, we consider \emph{data perturbing adversaries} that manipulate sampled data points, and \emph{distribution perturbing adversaries} that manipulate the data generating distribution itself. We show that these two adversaries are closely related and establish the precise relationship between them.
We primarily focus on the binary classification setting under $0$-$1$ loss function. 
We answer Question~\ref{q_1} by providing algorithm-independent bounds for adversarial risk that are agnostic to the classifier. 
We answer Question~\ref{q_2} by deriving the optimal adversarial classifier in some special settings, and by providing bounds on the deviation of the optimal adversarial classifier from the standard optimal classifier in more general settings.
Our contributions are listed below.

\begin{itemize}
    \item[(1)] We resolve Question~\ref{q_1} in the binary classification with 0-1 loss setting by deriving a formula for the optimal adversarial risk in terms of an optimal transport cost between the two data distribution. Our proof is novel and simple and connects adversarial machine learning to well-known results in optimal transport theory.
    %We provide a new and simple proof for the characterization of optimal adversarial risk for $0$-$1$ loss functions in terms of an optimal transport cost between the two data generating distributions, where the transport cost is given by $c_\epsilon(x, x') = \1\{d(x,x')>2\epsilon\}$ and $\epsilon$ is the perturbation budget of the adversary. This is analogous to the optimal risk (i.e., Bayes risk) for binary classification in the standard setting, which is a function of the total variation distance between the two data generating distributions, which in turn is also an optimal transport cost with transport cost $c_0(x, x') = \1\{d(x,x')>0\}$. This completely answers Question~\ref{q_1} for this setting. Our proof establishes connections between adversarial machine learning and well-known results in the theory of optimal transport. 

    \item[(2)] We construct optimal couplings for the optimal transport cost from (1) when the two data distributions are univariate normal, uniform over intervals, and triangular. We resolve Question~\ref{q_2} in these cases by determining the optimal adversarial classifiers using the optimal couplings. Our results indicate that the decision boundary can be sensitive to the adversary's budget.
    %We propose a novel coupling strategy that achieves the proposed optimal transport cost between the two class-conditional densities for several univariate distributions like Gaussian, triangular and uniform. Using the analysis of optimal couplings, we obtain the optimal adversarial classifiers for these settings. This answers Question~\ref{q_2} in these settings, and shows how the decision boundary of the optimal classifier changes in the presence of adversaries. In certain cases, we show that the decision boundary can change arbitrarily, even for small changes in the adversary budget $\epsilon$.

    \item[(3)] We calculate the optimal adversarial risk for the CIFAR10, MNIST, Fashion-MNIST, and SVHN datasets. We perform a similar calculation for data-augmented versions of these datasets. The non-zero values resulting from these calculations highlight the impossibility of being completely accurate---even on the training set---in adversarial settings.
    %Using our analysis for $0$-$1$ loss, we obtain the exact optimal risk attainable for a range of adversarial budgets under $\ell_2$-norm and $\ell_\infty$-norm perturbation of data, for several real-world datasets, namely CIFAR10, MNIST, Fashion-MNIST, and SVHN. 
    %In addition, we analyze Gaussian mixture distributions based on these datasets and compute lower bounds on the optimal adversarial risk  for them. These bounds indicate the optimal adversarial error  achievable with data augmentation using Gaussian perturbations.
        
	\item[(4)] We partially address Question~\ref{q_1} for continuous loss functions by deriving upper and lower bounds on the optimal adversarial risk which depend on convexity and smoothness assumptions of the loss with respect to data. We also partially address Question~\ref{q_2} by upper bounding how much the optimal hypothesis with an adversary can deviate from the optimal hypothesis without an adversary. These bounds are in terms of the curvature of the loss function with respect to the parameters of the hypotheses.
	%Addressing Questions~\ref{q_1} and \ref{q_2} is more challenging for general continuous loss functions. We provide some preliminary results to address these that include upper and lower bounds on the optimal adversarial risk. These bounds depend on the convexity and smoothness of the loss function with respect to the data. For Question~\ref{q_2}, we prove a result that quantifies the maximum amount by which the optimal classifier can deviate in the presence of an adversary, in terms of the curvature of the loss function at the optimal classifier for standard risk.
        
\end{itemize}

\paragraph{Structure:}  The rest of the paper is structured as follows: In Section~\ref{sec: preliminaries}, we introduce the data-perturbing and distribution-perturbing models of adversaries and show that the data-perturbing adversary may be considered to be a special case of the distribution-perturbing adversary. We also discuss related work, especially, other related notions of adversaries studied in the literature.
In Section~\ref{sec: opt risk 01 loss}, we discuss the optimal adversarial risk for binary classification with $0$-$1$ loss. We settle Question~\ref{q_1} in this setting by introducing the $D_\epsilon$ optimal transport cost that completely characterizes the optimal risk.
In Section~\ref{sec: couplings} we present a coupling strategy that achieves the optimal transport cost in special cases of interest in the univariate case. Using this coupling, we obtain the optimal adversarial classifier, thus settling Question~\ref{q_2} for these special cases.
In Section~\ref{sec: opt risk gen loss}, we discuss the optimal risk for general loss functions and present our bounds on the optimal adversarial risk.
In Section~\ref{sec: opt classifier}, we discuss optimal classifiers for general loss functions and present our deviation bounds on the optimal adversarial classifier.
Finally, in Section~\ref{sec: expts}, we present adversarial risk lower bounds for real world datasets and evaluate our bounds for $0$-$1$ loss function.

\paragraph{Notation:}
The complement of a set $A$ is denoted by $A^c$. The  indicator function that maps all the inputs satisfying condition $C$ to $1$ and the rest to $0$ is denoted by $\1{\{C\}}$. 
The set of probability measures over the measure space $(\cX, \sigma(\cX))$, where $\cX$ is a Polish space and $\sigma(\cX)$ is the Borel sigma algebra over $\cX$, is denoted by $\cP(\cX)$.
For any two probability measures $\mu, \nu\in \cP(\cX)$, the set of all joint probability measures (or couplings) over $\cX\times\cX$ with marginals $\mu$ and $\nu$ is denoted by $\Pi(\mu, \nu)$.
The total variation distance and $p$-Wasserstein distance between $\mu$ and $\nu$ is denoted by $D_{TV}(\mu, \nu)$ and ${W_p}(\mu, \nu)$, respectively. A norm and its dual are denoted by $\lVert\cdot\rVert$ and $\lVert\cdot\rVert_*$, respectively. The cumulative distribution function (cdf) of the standard normal distribution is denoted by $\Phi$ and its tail distribution is denoted by $Q(x) := 1- \Phi(x)$.

\section{Models in adversarial machine learning}\label{sec: preliminaries}

%In this section, we introduce the problem setup of learning in the presence of an adversary. In Section~\ref{sec: optimal_transport}, we introduce some definitions from optimal transport that will be used throughout the paper. In section~\ref{sec: types_of_adversaries} we introduce two notions of adversaries and compare them with existing notions in the literature.

We describe models of adversaries that are commonly invoked in machine learning and highlight connections between them. We use the following convention: Let $(\cX, d)$ denote a separable Hilbert space with metric $d$ for the data points and $\cY$ denote the finite set of discrete labels assigned to the data-points. Let $\rho$ be the data distribution which we express as $\rho_y(y)\rho_{x|y}(x)$, where $\rho_y(y)$ is the marginal probability of label $y \in \cY$ and $\rho_{x|y}(\cdot)$ is the conditional distribution of $X$ given $Y = y$. Let the hypothesis class be $\cW$. Let $\ell: (\cX\times \cY)\times\cW\to \mathbb{R}^+$ denote a loss function such that $\ell((\cdot, \cdot), w)$ is $\rho$-measurable for all $w\in \cW$. Let $\cZ \defn \cX \times \cY$ and $z \defn (x,y) \in \cZ$.

\subsection{Types of adversaries: Informal description}

To quantify the impact of an adversary, several notions of adversarial risk have been proposed in the literature. We highlight two most popular notions: (i) adversary perturbs data points, and (ii) adversary perturbs data distributions. 

\paragraph{Data perturbing adversary:}

A data-perturbing adversary of budget $\epsilon$ can perturb $x \in \cX$ to any $x'\in \cX$ such that $d(x,x')\leq \epsilon$. The adversary wishes to maximize loss, and so would choose $x'$ accordingly. A natural definition for adversarial loss (or robust loss) incurred by a hypothesis $w\in \cW$ for an adversary with budget $\epsilon$ is \citep{MadEtal18, ShaEtal18}:
\begin{align}\label{eq: data_loss}
    R_\epsilon(\ell, w) = \E_{(x,y)\sim\rho_y\rho_{x|y} }\left[\sup_{d(x,x')\leq\epsilon}\ell((x', y),w)\right].
\end{align}

\paragraph{Distribution perturbing adversary:} 
The adversarial loss incurred by a hypothesis $w\in \cW$ in the presence of a distribution perturbing adversary with a budget $\epsilon$ is defined as follows:
\begin{align}\label{eq: distr_loss1}
    \widehat R_\epsilon(\ell, w) = \sup_{\rho' \in B_\epsilon(\rho)} \E_{z \sim \rho'} \ell(z,w),
\end{align}
where $B_\epsilon(\rho)$ may be thought of as a ball of radius $\epsilon$ around $\rho$, the true data generating distribution. The Wasserstein distance has been one of the more popular metrics used to define $B_\epsilon(\cdot)$ in the space of distributions \citep{Woz14, BlaMur16, GaoKle16, GaoEtal17, EsfKuh18, ZhuEtal19}. 

\subsection{Types of adversaries: Formal description}

\paragraph{Data perturbing adversary:} 
Formulation~\eqref{eq: data_loss} is adequate for most practical purposes, but runs into measure-theoretic difficulties due to the arbitrary choice involved in the adversary's perturbations. Even if the adversarial map $x\to x'$ is $\rho$-measurable, the function $\sup_{d(x,x')\leq\epsilon}\ell((x', y),w)$ may not be $\rho$-measurable. Moreover, the adversary may not use a deterministic mapping to perturb the data points, and rather do so randomly.
In light of these considerations, we redefine a data perturbing adversary to be a collection of Markov kernels indexed by $y \in \cY$ denoted by $\kappa_y: \cX\times\sigma(\cX)\to [0,1]$. Equivalently, for each $y \in \cY$ the kernel $\kappa_y$ satisfies:
\begin{enumerate}
\item For all $x\in\cX$, the map $A\to \kappa_y(x, A)$ is a probability measure denoted by $\kappa_{y,x}\in \cP(\cX)$,
\item For all $A\in \sigma(\cX)$, the map $x\to \kappa_{y,x}(A)$ is measurable. 
\end{enumerate}
 Let the collection of kernels indexed by $y$ be $\kappa \defn \{\kappa_y\mid y\in \cY\}$. It is useful to think of $\kappa_{y,x}$ as the adversary's perturbation strategy after observing the sample $x$ and its label $y$ -- the adversary perturbs $x$ to $x'$ where the latter is the result of passing $x$ through the Markov kernel $\kappa_{y,x}$. The collection of kernels $\kappa$ completely describes  the adversary's strategy. %For a true data distribution $\chi\in \cP(\cX)$,
%we define a joint probability measure, $\pi_{\chi, \kappa} \in \cP(\cX\times \cX)$ as follows. For $A,B\in \sigma(\cX)$,
%\begin{align*}
%\pi_{\chi, \kappa}(A\times B) = \int_{x\in A} \kappa_x(B) d\chi.
%\end{align*}
We use $\rho^\kappa(x,y,x')$ to denote the joint distribution of $(x,y,x')$ induced by $\kappa$. Let the joint distribution of $(x, x')$ conditioned on $y$ be denoted by $\rho^\kappa_{(x,x')|y} \in \cP(\cX \times \cX)$, and the conditional distribution of $x'$ given $y$ be denoted by $\rho^\kappa_{x'|y} \in \cP(\cX)$.
We say that the adversary $\kappa$ has a budget of $\epsilon\geq 0$, denoted by $\kappa \in K_\epsilon$, if the following holds $\rho_y$-almost surely:
\begin{align}\label{eq: adv_constraint}
	\esssup_{(x, x') \sim \rho^\kappa_{(x,x')|y}} d(x,x') \leq \epsilon.
\end{align}
%We note that the constraint in \eqref{eq: adv_constraint} allows the adversary $\kappa$ to perturb data points from a $\chi$-null set by an arbitrary amount. 
%The distribution of $(x,y,x')$ for the $\kappa$ adversary is $\chi(x)\psi_x(y)\kappa_x(x')$, which we denote by $\chi\psi_x\kappa_x$. 
%Let $K_\epsilon$ denote the set of all Markov kernels over $\cX\times \sigma(\cX)$ satisfying the constraint in \eqref{eq: adv_constraint}. 
Then, we have the following definition for the adversarial risk incurred by a hypothesis $w\in \cW$ in the presence of a data perturbing adversary of budget $\epsilon$:
\begin{align}\label{eq: data_loss2}
	R_\epsilon(\ell, w) = \sup_{\kappa\in K_\epsilon} \E_{(x,y, x')\sim \rho^\kappa(x,y,x')} \left[ \ell((x', y), w) \right].
\end{align}
We will use the definition in \eqref{eq: data_loss2} rather than the one in \eqref{eq: data_loss} to denote the adversarial loss unless specified otherwise. 
 
The optimal adversarial loss attainable over the hypotheses $w\in \cW$ is defined as the {\em optimal adversarial risk} or {\em optimal robust risk},
\begin{align}\label{eq: optimal distr. robust risk}
    R^*_\epsilon = \inf_{w\in\cW} R_\epsilon(\ell, w).
\end{align}
The hypothesis attaining the optimal adversarial risk (if it exists) is called the {\em optimal adversarial hypothesis} and is denoted by $w^*_\epsilon$. Note that for $\epsilon=0$, we have $x=x'$ almost surely, and so the adversarial risk equals Bayes risk,
$R^*_0 = \inf_w \E_{z\sim\rho}\left[ \ell((x,y),w) \right]$.

\paragraph{Distribution perturbing adversary:}

Since distribution perturbing adversaries considered in this paper rely on optimal transport distances, we introduce optimal transport briefly. Let $\mu, \nu\in \cP(\cX)$ and let $c: \cX\times\cX\to\mathbb{R}^+$ denote the cost $c(x,x')$ of transporting unit mass from $x\in\cX$ to $x'\in\cX$. The optimal transport cost between $\mu$ and $\nu$ is given by,
\begin{align}\label{eq: OT_cost}
	\cT_c(\mu, \nu) = \inf_{\pi\in\Pi(\mu, \nu)}\E_{(x,x')\sim\pi} c(x, x'),
\end{align}
where $\Pi(\mu, \nu)$ is the set of all couplings between $\mu$ and $\nu$. When $c(x,x')=d(x,x')$, where $d$ is a metric over $\cX$, the optimal transport cost is the $1$-Wasserstein distance; i.e., $W_1(\mu, \nu) = \cT_d(\mu, \nu)$. For $p\geq 1$, the $p$-Wasserstein distance is given by $W_1(\mu, \nu) = \left(\cT_{d^p}(\mu, \nu)\right)^p$. The $\infty$-Wasserstein distance is defined to be the limit of the $p$-Wasserstein distances as follows: $W_\infty(\mu, \nu) = \lim_{p\to\infty} W_p(\mu, \nu)$. Alternatively, the $\infty$-Wasserstein distance is also characterized as follows \citep{GivSho84}.
\begin{align}
W_\infty(\mu, \nu) &= \inf\{\delta > 0~:~\mu(A) \leq \nu(A^\delta)\text{ for all measurable~} A\}\label{eq: W_infty_def_sets}\\
&= \inf_{\pi\in\Pi(\mu, \nu)} \esssup_{(x,x')\sim\pi} d(x, x') \label{eq: W_infty_def_esssup}.
\end{align}
For $1\leq p\leq q$, we have $W_p(\mu, \nu)\leq W_q(\mu, \nu)\leq W_\infty(\mu, \nu)$. Hence, the $W_\infty$-metric is stronger than any $W_p$-metric. 

%For each $y \in \cY$, let $B_\epsilon(\rho_y) = \{\mu \in \cP(\cX) \mid W_p(\mu, \rho_y) \leq \epsilon}$. Let the collection of all these balls be $B_\epsilon = \{B_\epsilon(y) \mid y \in \cY\}$. 
An adversary $\gamma$ is a collection of distributions over $\cX$ indexed by $y$; i.e., $\gamma = \{\rho^\gamma_{x'|y}\in \cP(\cX) \mid y\in \cY\}$. An adversary $\gamma$ is said to have budget $\epsilon$ in the $p$-Wasserstein space if, the following holds $\rho_y$-almost surely: 
\begin{align*}
W_p(\rho^\gamma_{x|y}, \rho_{x|y}) \leq \epsilon.
\end{align*}
This is denoted by $\gamma \in \Gamma^p_\epsilon$.
Note that $\Gamma^q_\epsilon \subseteq \Gamma^p_\epsilon$ for $1\leq p\leq q \leq \infty$.
It is useful to think of $\rho^\gamma_{x'|y}$ as the adversary's strategy after observing the sample $(x,y)$. The adversary perturbs $x$ to $x'$ such that $x' \sim \rho^\gamma_{x'|y}$ and $x$ and $x'$ are conditionally independent given the label $y$. The collection of distributions $\{\rho^\gamma_{x'|y} \mid y\in \cY\}$ completely describes the adversary's strategy. The data distribution after the adversary's action is $(x', y) \sim \rho_y(y)\rho^\gamma_{x'|y}$.
Let $\rho^\gamma(x',y)\in \cP(\cX\times \cY)$ denote this distribution.
For a loss function $\ell$ and hypothesis $w$, the adversarial risk is defined as
\begin{align}\label{eq: distr_loss1_formal}
    \widehat R_\epsilon^p(\ell, w) = \sup_{\gamma \in \Gamma_\epsilon^p} \E_{(x',y) \sim \rho^\gamma(x',y)} \ell((x',y),w).
\end{align}

The optimal adversarial loss attainable over the hypotheses $w\in \cW$ is defined as the {\em optimal adversarial risk} or {\em optimal robust risk},
\begin{align*}
   \widehat R^{p,*}_\epsilon = \inf_{w\in\cW} \widehat R^p_\epsilon(\ell, w).
\end{align*}
The hypothesis attaining the optimal adversarial risk (if it exists) is called the {\em optimal adversarial hypothesis} and is denoted by $\widehat w^*_\epsilon$. Note as before that  for $\epsilon=0$, adversarial risk equals Bayes risk.

\paragraph{Relation between the two types of adversaries:}
Our goal is to show that the data perturbing adversary is a special case of the distribution perturbing adversary in the following sense: The adversarial loss incurred under $K_\epsilon$ is identical to the adversarial loss incurred under a $\Gamma^\infty_\epsilon$. (Note that the adversaries themselves are different due to the conditional independence of $x$ and $x'$ given $y$ for the distribution perturbing adversary.) This is shown in the following theorem.

\begin{theorem}\label{thm: two_types_of_adv}
Let $\cX$ be a separable Hilbert space, let $\cY$ be a discrete set of labels, let $\cW$ be a hypothesis class, and let $\ell: \cZ \times \cW \to \real_+$ be a $\rho$-measurable loss function. Let 
$$R_\epsilon(\ell, w) = \sup_{\kappa\in K_\epsilon} \E_{(x,y,x')\sim \rho^\kappa(x, y, x')} \left[ \ell((x', y), w) \right]$$
and 
$$\widehat R^\infty_\epsilon(\ell, w) = \sup_{\gamma\in \Gamma^\infty_\epsilon}  \E_{(x',y)\sim \rho^\gamma(x',y)}\left[\ell((x, y),w)\right].$$ Then $R_\epsilon(\ell, w) = \widehat R_\epsilon^\infty(\ell, w)$. 
%\begin{align}
%	R_\epsilon(\ell, w)
%	= \sup_{\kappa\in K_\epsilon} \E_{(x,y,x')\sim \rho_{x|y}\rho_y\kappa_{y,x}} \left[ \ell((x', y), w) \right]
%	= \sup_{\gamma\in \Gamma^\infty_\epsilon}  \E_{(x',y)\sim \rho'_{x|y}\rho_y}\left[\ell((x, y),w)\right].
%\end{align}
\end{theorem}
\begin{proof} 
Let $\kappa\in K_\epsilon$.  Observe that the adversarial loss depends only on the marginal distribution of $(x',y)$ in the joint distribution of $(x,y, x')$ induced by $\kappa$. 
Recall that we denote the joint distribution of $(x, x')$ conditioned on $y$ by $\rho^\kappa_{(x,x')|y} \in \cP(\cX \times \cX)$, and the conditional distribution of $x'$ given $y$ by $\rho^\kappa_{x'|y} \in \cP(\cX)$. %, which is given by
%\begin{align*}
%\rho'(x',y) = \int_{x \in \cX} \rho(x,y)\kappa_{y,x}(x') dx.
%\end{align*}
%Express $\rho'(x,y) = \rho'_y(y)\rho'_{x|y}(x)$. 
%We consider an alternate adversary $\kappa'$ such that $\kappa'_{y,x} = \rho^\kappa_{x'|y}(\cdot)$; i.e., the kernel $\kappa'_{y,x}$ depends only on $y$. Observe that $\kappa$ and $\kappa'$ lead to the same joint distribution of $(x',y)$ and so the adversarial loss is identical for both.

Define  a distribution perturbing adversary $\gamma$ as follows: $\gamma = \{\rho^\gamma_{x'|y} := \rho^\kappa_{x'|y} \  |\ y\in \cY \}$. The key point to note here is that the marginal distribution $\rho^\gamma(x',y)$ is identical to $\rho^\kappa(x',y)$, and so the adversarial risk for both adversaries is identical. For $y\in \cY$,
\begin{align*}
W_\infty(\rho_{x|y}, \rho^\gamma_{x'|y}) &= W_\infty(\rho_{x|y}, \rho^\kappa_{x'|y})
\leq \esssup_{(x,x')\sim \rho^\kappa_{(x, x')|y}} d(x,x')
\leq \epsilon.
\end{align*}
Here the first inequality follows since $W_\infty$ is the infimum of the essential supremum over all couplings and $\rho^\kappa$ is just one such coupling, and the final inequality follows since $\kappa \in K_\epsilon$. This shows that for every $\kappa \in K_\epsilon$, there is a $\gamma \in \Gamma^\infty_\epsilon$ that achieves the same adversarial risk, and so we conclude $R_\epsilon(\ell, w) \leq \widehat R_\epsilon^\infty(\ell, w).$

We will now prove the above inequality in the reverse. 
Let $\gamma= \{\rho^\gamma_{x'|y}\  |\ y\in \cY \}\in \Gamma^\infty_\epsilon$. 
Then, $W_\infty(\rho_{x|y},\rho^\gamma_{x'|y})\leq \epsilon$ for all $y\in \cY$.
Fix a $y\in \cY$. By the definition of $W_\infty$, we have that for any positive integer $n$, there exists a joint probability measure $\pi_n\in \Pi(\rho_{x|y}, \rho^\gamma_{x'|y})$ such that  $\esssup_{(x,x')\sim \pi_n} d(x,x')<\epsilon+ 1/n$.
We now show that the sequence of measures $\pi_n$ is tight. Given a $\delta>0$, let $E\subseteq \cX$ be a compact set such that $\min\{\rho_{x|y}(E), \rho^\gamma_{x'|y}(E)\}>1-\delta/2$. Then, 
\begin{align*}
\pi_n((E\times E)^c) \leq \rho_{x|y}(E^c) + \rho^\gamma_{x'|y}(E^c) <\delta.
\end{align*}
Hence, by Prokhorov's theorem (for reference, see Theorem 5.1 in \citep{Bil13}), there is a
subsequence of $(\pi_n)$ that converges weakly to a
 probability measure $\pi^* \in \Pi(\rho_{x|y}, \rho^\gamma_{x'|y})$ that satisfies,
\begin{align*}
\esssup_{(x,x')\sim \pi^*} d(x,x')\leq\epsilon.
\end{align*}
Since $\cX$ is a complete and separable space, there exists a Markov kernel $\lambda_y: \cX\times \sigma(\cX)\to [0,1]$ such that for any $A,B\in \sigma(\cX)$, we have
\begin{align*}
\pi^*(A\times B) = \int_{x\in A} \lambda_{y,x}(B) d\rho_{x|y}.
\end{align*}
The existence of such a probability kernel is guaranteed by the product-restricted conditional probability property of Radon spaces (see Theorem 3.1 in \citep{LeaEtal04}). 
Repeating the above argument, we construct a kernel $\lambda_y$ for each $y\in \cY$ satisfying $\esssup_{(x,x')\sim \rho_{x|y}\lambda_{y,x}} d(x,x')\leq \epsilon$. Then, $\lambda := \{\lambda_y \ | \ y\in \cY\} \in K_\epsilon$. Moreover, the joint distribution of $(x',y)$ under $\rho^\gamma_{(x',y)}$ is identical to $\rho^\lambda_{(x',y)}$, and so the corresponding adversarial risks are identical.
Hence, for every $\gamma \in \Gamma^\infty_\epsilon$ there exists a $\lambda \in K_\epsilon$ that achieves an identical adversarial risk. This leads to the conclusion $\widehat R_\epsilon^\infty(\ell, w) \leq R_\epsilon(\ell, w)$, which completes the proof.
\end{proof}

For any $p \geq 1$, a distribution perturbing adversary of budget $\epsilon$ in $p$-Wasserstein space is   more powerful than a data perturbing adversary of the same budget, as shown by the following corollary:
\begin{corollary}\label{thm: stronger adversary}
For $p \geq 1$, consider a $W_p$-data-perturbing adversary of budget $\epsilon$. The following inequality holds:
\begin{align}
R_\epsilon(\ell, w) \leq \widehat{R}_\epsilon^p(\ell, w),
\end{align}
for all $w\in \cW$. Moreover, $R^*_\epsilon \leq \widehat R^{p,*}_\epsilon$.
\end{corollary}
\begin{proof}
We have the following sequence of inequalities:
\begin{align}
R_\epsilon(\ell, w)
&= 
\sup_{\gamma\in \Gamma^\infty_\epsilon}  \E_{(x',y)\sim \rho^\gamma{(x',y)}}\left[\ell((x', y),w)\right]\label{eq: adv1}\\
&\leq \sup_{\gamma\in \Gamma^p_\epsilon}  \E_{(x',y)\sim \rho^\gamma{(x',y)}}\left[\ell((x', y),w)\right]\label{eq: adv2}\\
&= \widehat{R}_\epsilon^p(\ell, w)\nonumber,
\end{align}
where the equality in~\eqref{eq: adv1} follows from Theorem~\ref{thm: two_types_of_adv} and the inequality in~\eqref{eq: adv2} follows from the fact that $\Gamma^\infty_\epsilon\subseteq \Gamma^p_\epsilon$. Taking infimum over $w\in \cW$ on both sides of the inequality, we get $R^*_\epsilon \leq \widehat R^{p,*}_\epsilon$.
\end{proof}

To see that the inequality in Corollary~\ref{thm: stronger adversary} can be strict, consider the following example of binary classification with $0$-$1$ loss. Let $P(X|Y=0)$ (denoted by $p_0$) be a uniform distribution over $[0,1]$ and $P(X|Y=1)$ (denoted by $p_1$) be a constant distribution at $X=0$. Let both the classes be equally likely. In this case, $W_p(p_0, p_1) = \left(\int_0^1 x^p dx\right)^\frac{1}{p} = 1/(p+1)^\frac{1}{p}$. Hence, $W_1(p_0, p_1) = 1/2$, while $W_\infty(p_0, p_1) = 1$. Taking the adversarial budget to be $\epsilon = 3/4$, it is easy to see that $\widehat{R}^{1,*}_\epsilon = 1/2$ because both the distributions are within the perturbation budget of the $W_1$-distribution perturbing adversary. But $R^*_\epsilon = \epsilon/2 = 3/8 < \widehat{R}^{1,*}_\epsilon$, which is achieved by the optimal classifier which declares label $0$ on the set $[\epsilon, 1]$ and $0$ otherwise.

\paragraph{A remark on the risk bounds for adversaries:} All risk bounds proved in this paper are valid for both adversaries. Since the distribution perturbing adversary is stronger than the data perturbing adversary, any lower bound that holds for the latter holds for the former. Analogously, any upper bound for the distribution perturbing adversary holds for the data perturbing adversary.

%\subsection{Types of adversaries}\label{sec: types_of_adversaries}

\subsection{Related work}\label{sec: related work}

\paragraph{Related notions of adversarial risk:}

In this paper, we assume that the true data distribution $\rho(x,y)$ is expressed as $\rho_y(y)\rho_{x|y}(x)$. This model allows for randomness in the label $y$ for a fixed $x$. A special case of this model is when the existence of a true labelling function is assumed; i.e., there exists a function $c: \cX\to \cY$ such that $c(x)$ is the true label of $x$ for any $x\in \cX$. That is, $\rho(x,y) = \1\{y=c(x)\}\rho_x(x)$.
Under this model, Gourdeau et al. \citep{GouEtal19} define \emph{constant-in-the-ball risk} as
\begin{align}\label{eq: related_risk_1}
	R(h) = \mathbb{P}_{x\sim \rho_x}[\exists x': d(x,x')\leq \epsilon, \ h(x')\neq c(x)].
\end{align}
Rewriting in terms of expectation, we get the following.
\begin{align*}
	R(h) = \E_{x\sim \rho_x}\1\{\exists x': d(x,x')\leq \epsilon, \ h(x')\neq c(x)\}
	=  \E_{x\sim \rho_x}\left[\sup_{d(x,x')\leq \epsilon} \1\{ h(x')\neq c(x)\}\right].
\end{align*}
Hence,  the \emph{constant-in-the-ball risk} defined above is identical to adversarial risk defined in \eqref{eq: data_loss} for $0$-$1$ loss function under hypothesis $h$. The same notion of risk is also called the \emph{corrupted instance risk} in Diochnos et al. \citep{DioEtal18}.

A related notion of adversarial risk is the following:
\begin{align}\label{eq: related_risk_2}
R'(h) = \E_{x\sim \rho_x}\left[\sup_{d(x,x')\leq \epsilon} \1\{ h(x')\neq c(x')\}\right].
\end{align}
Here, the loss on the perturbed data point is evaluated with respect to the true label at the perturbed data point; i.e., $c(x')$ rather than the true label of the original data point $c(x)$. This notion of adversarial risk is termed \emph{exact-in-the-ball risk} in Gourdeau et al. \citep{GouEtal19} and \emph{error-region risk}  in Diochnos et al. \citep{DioEtal18}.
A key difference between $R(h)$ in \eqref{eq: related_risk_1} and $R'(h)$ in \eqref{eq: related_risk_2} is that $R'(h)$ is exactly equal to $0$ for $h=c$ for any $\epsilon\geq 0$ whereas $R(h)$ may be strictly positive even for $h=c$. Thus, the definition of $R'(h)$ allows for the existence of an optimal classifier whose adversarial risk is $0$, while the optimal classifier that minimizes $R(h)$ may still have a non-zero adversarial risk.
As noted in Gourdeau et al. \citep{GouEtal19}, $R(h)$ measures the sensitivity of the output label to corruptions in the input, while $R'(h)$ measures how well a hypothesis fits the ground-truth even with corrupted inputs.

Tu et al. \citep{TuEtal19} and Staib and Jegelka \citep{StaJeg17} use the adversarial risk of \eqref{eq: data_loss} to establish a version of 
Theorem~\ref{thm: two_types_of_adv}. However, they implicitly assume that there exists a measurable function that maps $x\to x'$ such that the supremum in \eqref{eq: data_loss} is attained for all $x\in \cX$. As explained previously, this is not true in general. 
Pinot et al. \citep{PinEtal20} define a notion of adversarial risk using a class of adversaries that use measurable maps to perturb data points as 
\begin{align*}
 R'_\epsilon(\ell, w) = \sup_{f\in \cF_\epsilon}\E_{(x,y)\sim \rho} [\ell((f(x), y),w)],
\end{align*}
where $\cF_\epsilon$ is the set of all $\rho$-measurable functions satisfying the budget constraint $\esssup_{x\sim \chi} d(x, f(x))\leq \epsilon$.
This definition is a special case of the definition in \eqref{eq: data_loss2}, where $K_\epsilon$ is restricted to the set of probability kernels with $k_x = \delta_{f(x)}$ (i.e. a probability measure with a single atom at $f(x)$) for $f\in \cF_\epsilon$.

\paragraph{Surrogates for the adversarial risk:}
Before adversarial deep learning, minimax risk was studied in the context of robust classification with linear classifiers and SVMs \citep{LanEtal02, ShiEtal06, XuEtal09, AbaEtal15, ChePas18}. Here, one proposes surrogate robust loss functions that can be tractably minimized. A similar strategy for minimizing adversarial loss may be found in~\citep{KhiLoh18}. For a discussion of surrogate losses, we refer the reader to Bao et al. \citep{BaoEtal20}.

In practice, the inner maximization term in the adversarial risk is approximated using gradient methods like the Fast Gradient Sign Method (FGSM) \citep{GooEtal14, CarWag17}. This gives rise to several related notions of risk that can be interpreted as a Taylor approximations for adversarial risk in definition \eqref{eq: data_loss}. \citep{MadEtal18, ShaEtal18}. 

Surrogate loss functions for ensuring Wasserstein distributional robustness have been proposed in~\citep{GaoKle16, EsfKuh18}, and robustness with respect to other optimal transport-based  perturbations is studied in~\citep{BlaMur19}. A key idea in these works is the dual formulation of optimal transport distances. As shown in Theorem~\ref{thm: two_types_of_adv}, adversarial robustness is equivalent to $W_\infty$-distributional robustness. However, the recent work on optimal transport-based robustness cannot be readily extended to the $W_\infty$-case because the $W_\infty$-metric does not admit a transport-cost minimizing formulation (for instance, compare \eqref{eq: OT_cost} with \eqref{eq: W_infty_def_esssup}) and so the classic Kantorovich-Rubinstein duality cannot be applied.
%Other related work on distributional robustness is discussed in the next paragraph.

\paragraph{Related notions of distributionally robust risk:}
The adversarial risk formulation under a distribution perturbing adversary has been widely studied in the distributionally robust optimization (DRO) literature \citep{GohSim10, WieEtal14}, with special focus on Wasserstein DRO~\citep{GaoKle16, EsfKuh18, BlaMur19}. The advantage of using Wasserstein metrics is the ability to  measure distances between probability distributions with non-overlapping supports, which is not possible for divergence-based measures.

The distributional uncertainty set is typically centered at the empirical distribution of the data points, unlike definition \eqref{eq: distr_loss1} where it is centered around the true data generating distribution. Bertsimas et al. \citep{BerEtal20} note that when the support of the true distribution is unbounded, the $W_\infty$-uncertainty set around the empirical distribution does not contain the true distribution for any $\epsilon$. Hence, $W_\infty$-distributional robustness is not considered in the distributional robustness setting, except for the works of Tu et al. \citep{TuEtal19} and Staib and Jegelka \citep{StaJeg17} that make a similar observation as our Theorem~\ref{thm: two_types_of_adv}. Distributionally robust risk has also been studied in a minimax statistical learning framework in \citep{LeeRag18, MazEtal20} for deriving generalization error bounds.

\paragraph{Connection to robust statistics:}
Finding optimal classifiers under $0$-$1$ loss is equivalent to hypothesis testing, and there are natural connections of adversarial machine learning to robust hypothesis testing. Classical literature on robust hypothesis studies robust versions of the likelihood ratio test under various (non-adversarial) contamination models such as Huber's $\epsilon$-contamination model, the total variation contamination model, or the Levy-Prokhorov metric contamination model \citep{Hub65, HubStr73}. Contamination models based on $f$-divergences have also been analyzed for the Kullback-Liebler divergence \citep{Lev08} and the squared Hellinger distance \citep{GulAbd13, GulZou17}.

For general loss functions, finding the parameters $w^* \in \mathcal{W}$ is akin to minimax robust estimation. Classical literature on minimax robust estimation studies problems such as density estimation and regression under a  parametrized uncertainty set of probability measures \citep{Hub04}. When the uncertainty sets are constructed with the Hellinger distance, methods are known for obtaining nearly optimal estimators \citep{Lec73, BarEtal17, BarBir18}.

\paragraph{Connection to concentration of measure:}
The concentration of measure phenomenon in high dimensional settings causes the measure of $\epsilon$-expansion of sets like $\mu(A^\epsilon)$ to blow up even for small $\epsilon>0$ \citep{Led01}.
Several authors suggest that adversarial examples are inevitable by appealing to concentration of measure phenomena in high dimensional~\citep{MahEtal19, GilEtal18, Doh19}. %This line of work aims to show that $R^*_\epsilon - R^*_0$ can be high even for small $\epsilon$.
Specifically, it has been shown that any classifier that works on data  sampled from concentrated metric probability spaces is susceptible to a high adversarial risk. For instance, when the input distribution is uniform over a high dimensional sphere \citep{GilEtal18} or Boolean hypercube \citep{DioEtal18}, or when the latent space of the data is a high dimensional Gaussian \citep{FawEtal18}, the adversarial risk can be significantly higher than standard risk, even for small $\epsilon$. 

\section{Optimal adversarial risk via optimal transport}\label{sec: opt risk 01 loss}
In this section, we present our results on adversarial risk under 
$0$-$1$ loss in the binary classification setting. 
We first define the optimal transport cost $D_\epsilon(\mu, \nu)$  between two probability measures $\mu$ and $\nu$ over a metric space $(\cX, d)$,  as follows. 
\begin{definition}[$D_\epsilon$ transport cost]
For $\epsilon \ge 0$, define the cost function $c_\epsilon: \cX \times \cX \to \real$ as $c_\epsilon(x, y) = \1\{d(x, y) > 2\epsilon\}$. The optimal transport cost $D_\epsilon$ is defined as
\begin{align}\label{eq_epsilon_Wasserstein}
    D_\epsilon(\mu, \nu) =
   \inf_{\pi\in\Pi(\mu, \nu)}\E_{(x,x')\sim\pi} c_\epsilon(x, x').
\end{align}
\end{definition}
\begin{remark*}
For $\epsilon = 0$, the optimal cost is equivalent to the total variation distance, i.e., $D_0(\mu, \nu) = D_{TV}(\mu, \nu)$. For $\epsilon > 0$, this cost does not define a metric over the space of distributions. This is because $D_\epsilon(\mu, \nu) = 0$ does not imply $\mu$ and $\nu$ are identical. Moreover, it also does not define a pseudometric since the triangle inequality is not satisfied. To see this, observe that if $\mu_1$, $\mu_2$, and $\mu_3$ are unit point masses at $0$, $2\epsilon$, and $4\epsilon$, then $D_\epsilon(\mu_1, \mu_3) = 1 > 0 = D_\epsilon(\mu_1, \mu_2) + D_\epsilon(\mu_2, \mu_3).$
\end{remark*}

Next, we present the main theorem of this section that gives the optimal risk under the binary classification setup for a data perturbing adversary.

\begin{theorem}\label{th_01bound}
Consider the binary classification setup with $\cY=\{0,1\}$, where the input $x\in\cX$ is drawn with equal probability from two  distributions $p_0$ (for label $0$) and $p_1$ (for label $0$). We consider a set of binary classifiers of the form $\1\{x\in A\}$, where $A\subseteq\cX$ is a topologically closed set.
That is, the classifier corresponding to $A$ assigns the label $1$ for all $x\in A$ and the label $0$ for all $x\notin A$.
Consider the $0$-$1$ loss function $\ell((x,y),A) = \1\{x\in A, y=0\} + \1\{x\notin A, y=1\}$. The adversarial risk with the data perturbing adversary is given by 
\begin{align}\label{eq: 01 opt risk}
    R^*_\epsilon
    =
    \frac{1}{2}\left[1- D_\epsilon(p_0, p_1)\right].
\end{align}
\end{theorem}
Instantiating Theorem~\ref{th_01bound} for $\epsilon=0$, we get 
$R^*_0 = \frac{1}{2}\left[1- D_{0}(p_0, p_1)\right] = \frac{1}{2}\left[1- D_{TV}(p_0, p_1)\right]$, which is the Bayes risk. It is also possible to derive weaker bounds in terms of the $p$-Wasserstein distance between the distributions of the two data classes, as shown in the following corollary:

\begin{corollary}\label{cor_dwp_bound}
Under the setup considered in Theorem~\ref{th_01bound},  we have the following bound for $p\geq 1$:
\begin{align}\label{eq_dwp_bound}
    R^*_\epsilon \geq
    \frac{1}{2}\left[1-  \left(\frac{{W_p}(p_0, p_1)}{2\epsilon}\right)^p   \right].
\end{align}
\end{corollary}

Our next result identifies a necessary and sufficient condition for $D_\epsilon(\mu, \nu) = 0$ for probability measures on a bounded support. When this holds, adversarial risk is $1/2$; i.e., no classifier can do better than random choice.

\begin{theorem}\label{thm: zerozero}
Let $\mu, \nu\in \cP(\cX)$. Then $D_\epsilon(\mu, \nu) = 0$ if and only if $W_\infty(\mu, \nu) \leq 2\epsilon$.
\end{theorem}

\subsection{Proofs of Theorems~\ref{th_01bound} and \ref{thm: zerozero}}

\begin{proof}[Proof of Theorem~\ref{th_01bound}]

Let $A \subseteq \cX$ be a closed set such that the classifier declares $1$ on $A$ and $0$ on $A^c$. Suppose the true hypothesis is $0$ and the observed sample is $x$. If the set $B(x, \epsilon) \defn \{y \in \cX \mid d(x,y) \leq \epsilon\}$ has a non-empty intersection with $A$, then the adversary can push $x$ to $x' \in B(x, \epsilon) \cap A$ such that the $\ell((x',0), A) = 1$. The set of all such $x$'s is given by:
\begin{equation}
A^{\oplus \epsilon} \defn \cup_{x \in A} B(x, r). 
\end{equation}
An equivalent way to express this is using Minkowski sums:
\begin{align*}
A^{\oplus \epsilon} = \{a + b \mid a \in A, b \in B(0, \epsilon)\}.
\end{align*}
It is easy to see that $(A^{\oplus \epsilon_1})^{\oplus \epsilon_2} = A^{\oplus (\epsilon_1+\epsilon_2)}$ for $\epsilon_1, \epsilon_2\geq 0$.

Similarly, if the true hypothesis is $1$, then the adversary can ensure a loss of 1 as long as $B(x, \epsilon) \cap A^c \neq \phi$. The set of all such $x$'s is given by
\begin{equation}
(A^c)^{\oplus \epsilon} \defn \cup_{x \in A^c} B(x, r). 
\end{equation}
We define $A^{\ominus \epsilon}$ as
\begin{equation}
A^{\ominus \epsilon} \defn ((A^c)^{\oplus \epsilon})^c.
\end{equation}

The robust risk over the hypothesis class of closed sets is given by
\begin{align*}
    R_\epsilon^*
    &= \inf_{A~\text{closed}}
    \frac{1}{2} \left( p_0(A^{\oplus \epsilon}) + p_1\left((A^c)^{\oplus \epsilon}\right)\right)\\
    &= 
    \frac{1}{2} \left( 1 - 
    \sup_{A~\text{closed}}
    \left\{
    p_1\left(A^{\ominus \epsilon}\right)-
    p_0(A^{\oplus \epsilon})
    \right\}
    \right).
\end{align*}
The following Lemma provides basic topological properties of the sets $A^{\oplus \epsilon}$ and $A^{\ominus \epsilon}$. 
\begin{lemma}[Proof in Appendix~\ref{app: lemma: oplus-ominus}]\label{lemma: oplus-ominus}
Let $\epsilon > 0$. If $A$ is a closed set, then $A^{\oplus \epsilon}$ and $A^{\ominus \epsilon}$ are also closed sets.
\end{lemma}

We now consider a slightly different notion of $\epsilon$-expansions of sets similar to our definition of $A^{\oplus \epsilon}$. For $\epsilon > 0$, define
\begin{align}
A^\epsilon = \{x \in \cX \mid d(x, A) \leq \epsilon\}, 
\end{align}
where 
\begin{align}
d(x, A) = \inf_{a \in A} d(x, a).
\end{align}
Our next Lemma shows the equivalence of $A^{\oplus \epsilon}$ and $A^{\epsilon}$ for closed sets $A$.

\begin{lemma}[Proof in Appendix~\ref{app: lemma: equivalence}]\label{lemma: equivalence}
For a closed set $A$, we have $A^\epsilon = A^{\oplus \epsilon}$.
\end{lemma}

Note that $A^{\oplus \epsilon}$ and $A^{\epsilon}$ need not be equivalent when $A$ is not closed. For example consider $(\cX, d) = (\real, |\cdot|)$ with $A = (0,1)$. Then, $A^{\oplus \epsilon} = (-\epsilon, 1+\epsilon)$ whereas $A^\epsilon = [-\epsilon, 1+\epsilon]$.

The main idea of our proof is to leverage Strassen's theorem (Appendix~\ref{app: strassen}), which states that
\begin{align*}
D_\epsilon(p_0, p_1) = \sup_{A~\text{closed}} \left\{ p_1(A) - p_0(A^{2\epsilon}) \right\}.
\end{align*}
To prove the equality  $R^*_\epsilon = \frac{1}{2}[1-D_\epsilon(p_0, p_1)]$, notice that it is enough to prove that 
\begin{align}\label{eq: pm_epsilon}
\sup_{A ~\text{closed}} p_1(A^{\ominus \epsilon}) - p_0(A^{\oplus \epsilon}) = \sup_{A ~\text{closed}} p_1(A) - p_0(A^{2\epsilon}).
\end{align}

We need the following lemma:

\begin{lemma}[Proof in Appendix~\ref{app: lemma: thick and thin}]\label{lemma: thick and thin fixed} 
Let $A$ be a closed set. Then $(A^{\ominus \epsilon})^{\oplus \epsilon} \subseteq A$ and $A \subseteq (A^{\oplus \epsilon})^{\ominus \epsilon}$. 
\end{lemma}
Figure~\ref{fig: square} illustrates the above lemma when $A$ is a square in $\real^2$ with the Euclidean distance metric.

We have the sequence of inequalities
\begin{align*}
\sup_{A ~\text{closed}} p_1(A) - p_0(A^{2\epsilon}) &\stackrel{(a)}\geq \sup_{A ~\text{closed}} p_1(A^{\ominus \epsilon}) - p_0((A^{\ominus\epsilon})^{2\epsilon})\\
&\stackrel{(b)}\geq \sup_{A ~\text{closed}} p_1(A^{\ominus \epsilon}) - p_0(A^{\oplus \epsilon}).
\end{align*}
Here, $(a)$ follows because $A^{\ominus \epsilon}$ is contained in the set of all closed sets by Lemma~\ref{lemma: oplus-ominus}. Inequality $(b)$ follows by the equivalence $(A^{\ominus \epsilon})^{2\epsilon} = (A^{\ominus \epsilon})^{\oplus 2\epsilon}$ from Lemma~\ref{lemma: equivalence}, and Lemma~\ref{lemma: thick and thin fixed} since 
$$(A^{\ominus \epsilon})^{2\epsilon} =  [(A^{\ominus \epsilon})^{\oplus \epsilon}]^{\oplus \epsilon} \subseteq A^{\oplus \epsilon},$$
and so $p_0((A^{\ominus \epsilon})^{2\epsilon}) \leq p_0(A^{\oplus \epsilon})$.

For the other direction, notice that
\begin{align*}
\sup_{A ~\text{closed}} p_1(A^{\ominus \epsilon}) - p_0(A^{\oplus \epsilon}) &\stackrel{(a)}\geq \sup_{A ~\text{closed}} p_1((A^{\oplus\epsilon})^{\ominus \epsilon}) - p_0((A^{\oplus \epsilon})^{\epsilon})\\
&\stackrel{(b)}\geq \sup_{A ~\text{closed}} p_1(A) - p_0(A^{2\epsilon}).
\end{align*}
Here, $(a)$ follows because $A^{\oplus \epsilon}$ is a closed set according to Lemma~\ref{lemma: oplus-ominus}. To see $(b)$, first note that using Lemma~\ref{lemma: equivalence},
\begin{align*}
(A^{\oplus \epsilon})^\epsilon &= (A^{\oplus \epsilon})^{\oplus \epsilon} = A^{\oplus 2\epsilon}
= A^{2\epsilon}.
\end{align*}
Thus, $p_0((A^{\oplus \epsilon})^\epsilon) = p_0(A^{2\epsilon})$. Moreover, Lemma~\ref{lemma: thick and thin fixed} states that
$A \subseteq (A^{\epsilon})^{\ominus \epsilon},$
and so $p_1(A) \leq p_1((A^{\oplus \epsilon})^{\ominus \epsilon})$. This completes the proof.

\begin{figure}[t]
    \centering
    \includegraphics[scale=0.6]{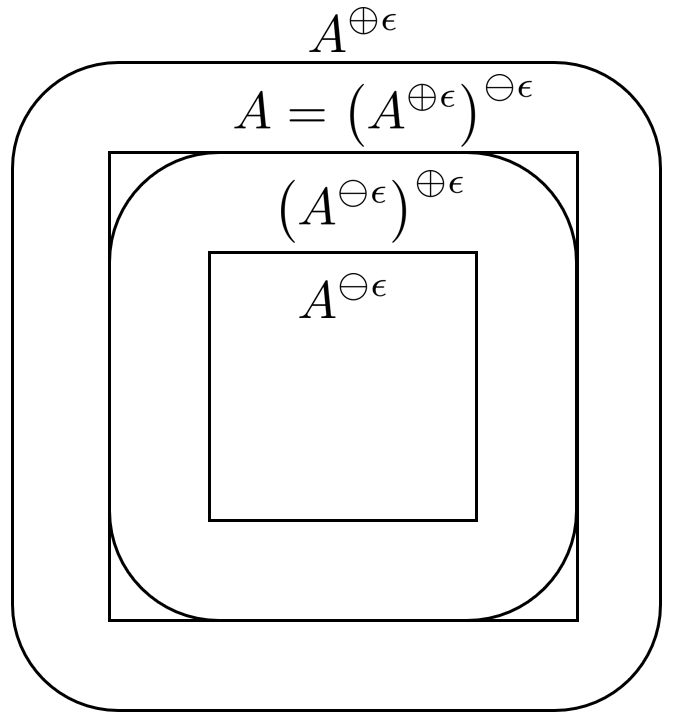}
    \caption{ Illustration of $A, A^{\oplus\epsilon}, A^{\ominus\epsilon}, (A^\oplus\epsilon)^{\ominus\epsilon}$, and $(A^{\ominus\epsilon})^{\oplus\epsilon}$ for a closed square in $(\mathbb{R}^2, \|\cdot\|_2)$. Observe that $(A^{\ominus\epsilon})^{\oplus\epsilon} \subseteq A$ and $A \subseteq (A^\oplus\epsilon)^{\ominus\epsilon}$.}
    \label{fig: square}
\end{figure}

\end{proof} 

\paragraph{Comparison with Bhagoji et al.~\cite{BhaEtal19}:} We point out that a similar result was obtained recently in \cite{BhaEtal19}. A key difference is that the proof in \cite{BhaEtal19} was established for a larger hypothesis class of measurable sets $A$; i.e., the following equality was established:
\begin{align*}
\sup_{A ~\text{measurable}} \mu(A^{\ominus \epsilon}) - \nu(A^{\oplus \epsilon}) = \sup_{A ~\text{measurable}} \mu(A) - \nu(A^{2\epsilon}).
\end{align*} 

It is not hard to check that $A^\epsilon$ is closed for any measurable set $A$, and so  
\begin{align*}
\sup_{A ~\text{measurable}} \mu(A) - \nu(A^{2\epsilon}) = \sup_{A ~\text{closed}} \mu(A) - \nu(A^{2\epsilon})
\end{align*}
We may restrict to the smaller hypothesis class of closed sets $A$ and use the result in~\cite{BhaEtal19} to obtain an inequality
\begin{align*}
\sup_{A ~\text{closed}} \mu(A^{\ominus \epsilon}) - \nu(A^{\oplus \epsilon}) \leq \sup_{A ~\text{closed}} \mu(A) - \nu(A^{2\epsilon}).
\end{align*}
Our result shows that this is, in fact, an equality.
\begin{proof}[Proof of Corollary~\ref{cor_dwp_bound}]

From Theorem~\ref{th_01bound}, we have
\begin{align*}
    R^*_\epsilon
    =
    \frac{1}{2}\left[1-
     \inf_{\pi\in\Pi(\mu, \nu)}\E_{(x,x')\sim\pi}[\1{\{d(x,x')>2\epsilon\}}]
      \right].
\end{align*}
For $p\geq 1$ and any $\pi\in\Pi(\mu, \nu)$, we have the following:
\begin{align*}
    \E_{(x,x')\sim\pi}[\1{\{d(x,x')>2\epsilon\}}]
    &= \E_{(x,x')\sim\pi}[\1{\{d(x,x')^p>(2\epsilon)^p\}}]
    \leq \E_{(x,x')\sim\pi}\left[\left( \frac{d(x,x')}{2\epsilon} \right)^p\right],
\end{align*}
where the last inequality follows from Markov's inequality.
Therefore, 
\begin{align*}
    R^*_\epsilon
    =
    \frac{1}{2}\left[1-
     \inf_{\pi\in\Pi(\mu, \nu)}\E_{(x,x')\sim\pi}\left[\left( \frac{d(x,x')}{2\epsilon} \right)^p\right]\right]
     \geq \frac{1}{2}\left[1-  \left(\frac{{W_p}(p_0, p_1)}{2\epsilon}\right)^p   \right].
\end{align*}
\end{proof}

\begin{proof}[Proof of Theorem~\ref{thm: zerozero}]
Since $W_\infty(\mu, \nu) = \inf\{\delta > 0\mid\mu(A) \leq \nu(A^\delta)\text{ for all measurable~} A\}$,  if $W_\infty(\mu, \nu) \leq 2\epsilon$, then $\mu(A) \leq \nu(A^{2\epsilon})$ for all closed sets $A$. Hence, 
\begin{align*}
D_\epsilon(\mu, \nu) = \sup_{A~\text{closed}} \mu(A) - \nu(A^{2\epsilon}) \leq 0.
\end{align*}
Since $D_\epsilon(\mu, \nu) \geq 0$, we conclude that $D_\epsilon(\mu, \nu) = 0$.

For the reverse direction, suppose that $D_\epsilon(\mu, \nu) = 0$. This means there exists a sequence of couplings $\{\pi\}_{i \geq 1}$ such that $\E_{\pi_i} c_\epsilon(x,x') \to 0$ where $\pi_i \in \Pi(\mu, \nu)$.  Using a strategy as in Theorem~\ref{thm: two_types_of_adv}, we conclude that the sequence $\{\pi_i\}$ is tight, and thus there exists a subsequence that converges weakly to $\pi^* \in \Pi(\mu, \nu)$. Since $c$ is a lower semicontinuous cost function, the coupling $\pi^*$ satisfies $\E_{\pi^*} c_\epsilon(x, x') = 0$, or equivalently, $\esssup_{(x, x') \sim \pi^*} d(x, x') \leq 2\epsilon.$ Using the definition of $W_\infty$ from \eqref{eq: W_infty_def_esssup}, we conclude $W_\infty(\mu, \nu) \leq 2\epsilon$. %Equivalently, the probability $\alpha_i := \prob_{\pi_i}(d(x,x') > 2\epsilon) \to 0$ as $i \to \infty$. For any fixed $p \geq 1$, we have the inequality
%\begin{align*}
%W_p(\mu, \nu) &\leq \left(\E_{\pi_i} d(x,x')^p\right)^{1/p}\\
%&\leq \left((2\epsilon)^p (1-\alpha_i) + C^p \alpha_i \right)^{1/p},
%\end{align*}
%where $C = \sup_{x,x'\in \cX} d(x,x')$ which is a constant since $\cX$ is bounded.
%Letting $i\infty$ in the above inequality, we get $W_p(\mu, \nu)\leq 2\epsilon$ for any fixed $p\geq 1$. Since $W_\infty(\mu, \nu) = \lim_{p\to\infty}W_p(\mu, \nu)$, we also have $W_\infty(\mu, \nu)\leq 2\epsilon$.

\end{proof}

\section{Optimal adversarial classifiers via optimal couplings}\label{sec: couplings}

In this section, we explicitly compute the optimal risk and optimal classifier for a data perturbing adversary in  some special cases. Instead of using $D_\epsilon$, we have shown in Corollary~\ref{th_01bound} that the optimal adversarial risk can be lower-bounded using other well-understood metrics such as the $W_p$ distances. However, these bounds are often too loose to use in practice, and this motivates us to study the optimal cost $D_\epsilon$ directly. In this section, we show that in certain special cases, the optimal coupling corresponding to calculating $D_\epsilon$ may be explicitly evaluated. Furthermore, in these cases, we can exactly characterize the optimal classifier and the optimal risk in the presence of an adversary. Given measures $\mu$ and $\nu$ corresponding to the two (equally likely) data classes, the general strategy we employ consists of the following steps: 
\begin{itemize}
\item[(1)] Propose a coupling $\pi$ between $\mu$ and $\nu$.
\item[(2)] Using this coupling, obtain the upper bound
\begin{align*}
D_\epsilon(\mu, \nu) \leq \E_{(x, x') \sim \pi} c_\epsilon(x, x').
\end{align*}
\item[(3)] Identify a closed set $A$ and compute a lower bound using %either 
%\begin{align*}
%D_\epsilon \geq \mu(A) - \nu(A^{2\epsilon}),
%\end{align*}
%or as per the proof of Theorem~\ref{th_01bound},
\begin{align*}
D_\epsilon(\mu, \nu) \geq \mu(A^{\ominus \epsilon}) - \nu(A^{\oplus \epsilon}).
\end{align*}
\item[(4)] Show that the lower and upper bounds match. This shows that the proposed coupling is optimal, and the sets $A$ and $A^c$ define the two regions of the optimal robust classifier.
\end{itemize}

In the examples we consider, guessing the set $A$ corresponding to the optimal robust classifier is easy. The challenging part is proposing a coupling and establishing its optimality. Although we shall focus on real-valued random variables, some of our results also naturally extend to higher dimensional distributions. 

In the following subsection, we review some results pertaining to optimal transport on the real line. We then present results that help in evaluating $D_\epsilon$ cost for real-valued random variables. In the subsequent subsections, we use these results to propose optimal couplings for several univariate distributions.

\subsection{Optimal transport on the real line}\label{sec: OT_real_line}

For a probability measure $\mu$ on $\real$, the cumulative distribution function (cdf) of $\mu$ is defined as $F(x) = \mu((-\infty, x])$, and for $t \in [0,1]$, the inverse cdf (or quantile function) is defined as $F^{-1}(t) = \inf\{x \in \real ~:~ F(x) \geq t\}$.

\begin{lemma}\label{lem: push_forward}[Theorem 2.5 in \citep{San15}]
Let $\mu$ and $\nu$ be probability measures on the real line, where $\mu$ is absolutely continuous with respect to the Lebesgue measure. Then there exists a unique non-decreasing function $T: \mathbb{R}\to\mathbb{R}$ such that $\mu(T^{-1}(A))=\nu(A)$ for any measurable set $A\subseteq \mathbb{R}$. Moreover, if $F$ and $G$ denote the cumulative distribution functions of $\mu$ and $\nu$ respectively, then $T$ is given by $T(x) = G^{-1}(F(x))$.
\end{lemma}
The function $T$ in Lemma~\ref{lem: push_forward} that transforms (or ``pushes forward") the measure $\mu$ into $\nu$ is called a \emph{monotone transport map}. Given a monotone transport map, we can define a coupling induced by the monotone map as follows. $(X, X')\sim \Pi(\mu, \nu)$ where $X\sim \mu$ and $X' = T(X)\sim \nu$. This coupling is also known by the name \emph{quantile coupling}.

The following lemma shows that the coupling induced by the monotone transport map is optimal for certain cases of the cost function.
\begin{lemma}[Theorem 2.9 in \citep{San15}]
Let $h:\mathbb{R}\to\mathbb{R}^+$ be a strictly convex function. Let $\mu$ and $\nu$ be probability measures on the real line, where $\mu$ has a density. Consider the cost function $c(x,x')=h(x'-x)$. Suppose $\cT_c(\mu, \nu)$ is finite. Then, $\cT_c(\mu, \nu) = \E_{x\sim\mu}[c(x, T(x))]$, where $T$ is the monotone transport map from $\mu$ to $\nu$.
\end{lemma}\label{lem: monotone_map_opt}

For the case of $h(x) = |x|^p$ where $p\geq 1$, Lemma~\ref{lem: monotone_map_opt} shows that the optimal coupling for $p$-Wasserstein distance is induced by the optimal transport map. However this may not be the case for $\infty$-Wasserstein distance. In the following theorem, we use the monotone map from Lemma~\ref{lem: push_forward} to present a more concrete condition than Theorem~\ref{thm: zerozero} for checking when $D_\epsilon(\mu, \nu) = 0$ for measures over $\real$.

\begin{theorem}\label{thm: zero}
Let $\mu$ and $\nu$ be probability measures on $\real$ that are absolutely continuous with respect to the Lebesgue measure with Radon-Nikodyn derivatives $f(\cdot)$ and $g(\cdot)$, respectively. Let $F$ and $G$ denote the cumulative distribution functions of $\mu$ and $\nu$ respectively.  Then $D_\epsilon(\mu, \nu) = 0$ if and only if $\|F^{-1} - G^{-1}\|_\infty \leq 2\epsilon$.
\end{theorem}
\begin{proof}
Consider the monotone transport map from $\mu$ to $\nu$ given by $T(x) = G^{-1}(F(x))$ as in Lemma~\ref{lem: push_forward}. We shall show that this map satisfies $|T(x) - x| \leq 2\epsilon$ for all $x \in \real$, and so the optimal transport cost $D_\epsilon$ must be 0. To see this, note that
\begin{align*}
T(x) - x &= G^{-1}(F(x)) - x\\
&\leq F^{-1}(F(x)) + 2\epsilon - x\\
&=  2\epsilon,
\end{align*}
where the last equality is in the $\mu$-almost sure sense. A similar argument shows $x-T(x) \leq 2\epsilon$, and thus $|T(x) - x| \leq 2\epsilon$. 

For the converse, suppose that there exists a $t_0 \in (0,1)$ such that $G^{-1}(t_0) - F^{-1}(t_0) > 2\epsilon$. Equivalently, $
G^{-1}(t_0) > F^{-1}(t_0) +
2\epsilon.$ Applying the $G$ function on both sides, 
\begin{align*}
t_0 > G(F^{-1}(t_0) +
2\epsilon).
\end{align*}
Consider the set $\tilde A = (-\infty, F^{-1}(t_0)]$. For this set, notice that 
\begin{align*}
\nu(\tilde A^{2\epsilon}) &= \nu((-\infty, F^{-1}(t_0)+2\epsilon])= G(F^{-1}(t_0) + 2\epsilon).
\end{align*}
Thus, we have
\begin{align*}
D_\epsilon(\mu, \nu) &= \sup_{A} \mu(A) - \nu(A^{2\epsilon})\\
&\geq \mu(\tilde A) - \nu(\tilde A^{2\epsilon})\\
&= t_0 - G(F^{-1}(t_0) + 2\epsilon)\\
&> 0.
\end{align*}
A similar argument may also be made for the case when $F^{-1}(t_0) - G^{-1}(t_0) > 2\epsilon$. 
\end{proof}
The above argument  shows that monotone transport maps are optimal when $D_\epsilon = 0$. But monotone maps are not always optimal for the cost function $c_\epsilon(\cdot, \cdot)$. Consider for example the two measures $\cN(0,1)$ and $\cN(1,1)$, and $\epsilon = 0.1$. The monotone map in this case is $T(x) = x+1$, which gives unit cost of transportation. However, Theorem~\ref{thm: same variance} shows that the optimal transport cost in this example is strictly smaller than 1. 

Checking the condition $\|F^{-1} - G^{-1}\| \leq 2\epsilon$ is not always easy. We identify a simple but useful characterization in the following corollary:
\begin{corollary}\label{cor: dominate}
Let $\mu$ and $\nu$ be as in Theorem~\ref{thm: zero}. Suppose that for every $x \in \real$, we have $F(x) \geq G(x)$ and $F(x) \leq G(x+2\epsilon)$. Then $D_\epsilon(\mu, \nu) = 0.$
\end{corollary}
\begin{proof}
Applying the $G^{-1}$ function to both sides of both inequalities, we arrive at
\begin{align*}
T(x) \geq x, \quad \text{ and } \quad T(x) \leq x+2\epsilon.
\end{align*}
This gives $\abs{T(x) - x} \leq 2\epsilon$ for all $x$, which concludes the proof.
\end{proof}

Theorem~\ref{thm: zero} may also be applied to finite positive measures $\mu, \nu$ with $\mu(\real)=\nu(\real)=U<\infty$ with simple scaling.
In what follows, we define a notion of optimal transport for finite positive measures that may have unequal masses.

\begin{comment}
\begin{definition}\label{def: dominates}
Let $\mu$ and $\nu$ be measures on $\real$ that are absolutely continuous with respect to the Lebesgue measure with a Radon-Nikodyn derivatives $f(\cdot)$ and $g(\cdot)$ respectively. For intervals $I$ and $J$, we denote $\mu_I$ to be the $\mu$ restricted to $I$ and $\nu_J$ to be $\nu$ restricted to $J$. For $a_1 < b_1$ and $a_2 < b_2$, we say that $\mu_{[a_1, b_1]} \succeq \nu_{[a_2, b_2]}$ if and only if 
\begin{align*}
\mu_{[a_1, b_1]}([a_1, a_1+t]) \geq \nu_{[a_2, b_2]}([a_2, a_2+t])
\end{align*}
for all $t \geq 0$. We say $\mu_{[b_1, a_1]} \succeq \nu_{[b_2, a_2]}$ if
\begin{align*}
\mu_{[a_1, b_1]}([b_1-t, b_1]) \geq \nu_{[a_2, b_2]}([b_2-t, b_2])
\end{align*} 
for all $t \geq 0$.
\end{definition}
\end{comment}

\begin{definition}\label{def: unequal}[Optimal transport cost for general measures]
Let $\mu$ and $\nu$ be as in Theorem~\ref{thm: zero}. Suppose that $\mu(\real) = U$ and $\nu(\real) = V$ and $U \leq V$. Let $\nu'$ be a measure on $\real$ with Radon-Nikodyn derivative $g'$ such that $\nu'(\real) = U$. We say $\nu' \subseteq \nu$, or $\nu'$ is contained in $\nu$, if $g(x) \geq g'(x)$ $\nu$-almost surely. Then the optimal transport cost $D_\epsilon(\mu, \nu)$ is defined as
\begin{align*}
D_\epsilon(\mu, \nu) =  \inf_{\nu' \subseteq \nu} D_\epsilon(\mu, \nu').
\end{align*}
Note that the amount of mass being moved is $\min(U,V) = U$.
\end{definition}

In the following two lemmas, we identify conditions under which $D_\epsilon(\mu, \nu)=0$ for finite positive measures with unequal mass.

\begin{lemma}\label{lemma: shift}
Let $\mu$ and $\nu$ be as in Theorem~\ref{thm: zero}. Assume that $\mu(\real) = U$ and $\nu(\real) = V$. Suppose the following conditions hold:
\begin{enumerate}
\item
The support of $g$ is a subset of $[a, +\infty)$ and the support of $f$ is a subset of $[a + 2\epsilon, +\infty) =: [a', +\infty)$. 
\item
For all $x \in \real$, we have $g(x) \leq f(x+2\epsilon)$.
\end{enumerate}
Then $D_\epsilon(\mu, \nu) = 0$. A similar result holds if the supports of $g$ and $f$ are subsets of $(-\infty, -a]$ and $(-\infty, -a-2\epsilon]$ respectively, and $f(-x-2\epsilon) \geq g(-x)$.
\end{lemma}

\begin{figure}[t]
    \centering
    \includegraphics[scale=0.3]{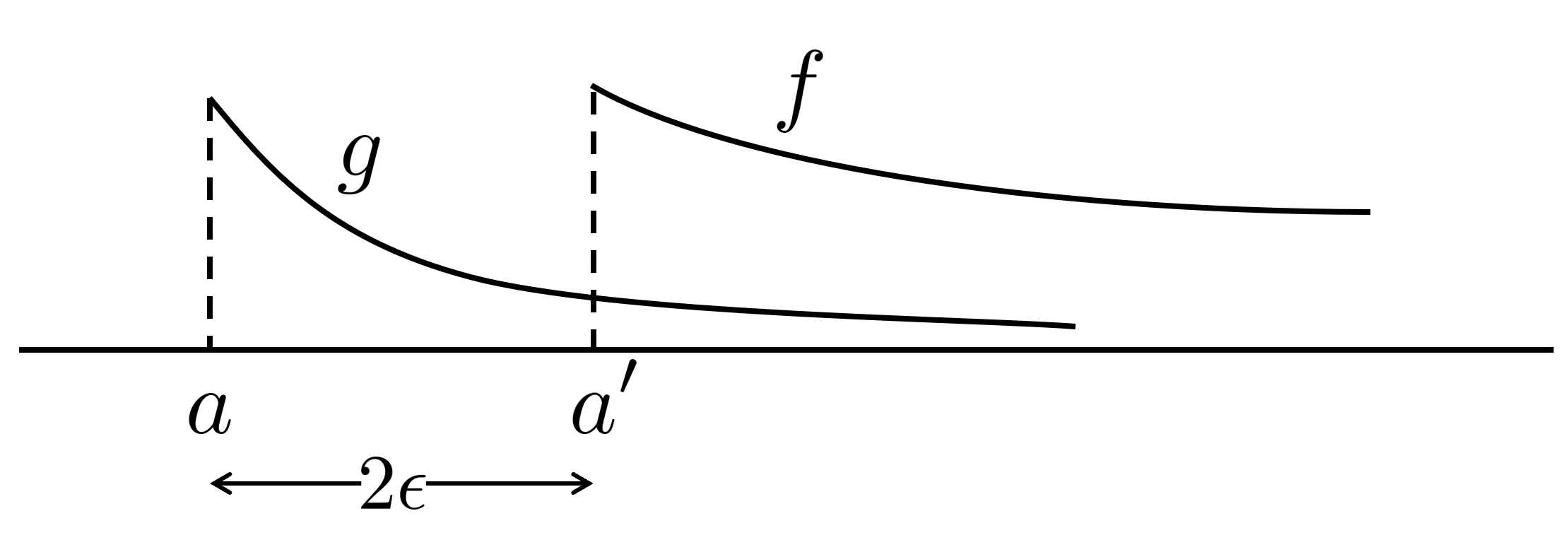}
    \caption{Figure illustrating the conditions in Lemma~\ref{lemma: shift}. }
    \label{fig:lemma shift}
\end{figure}

\begin{proof} Consider the transport map $T(x) = x+2\epsilon$ applied to $\nu$. This map has the effect of ``translating'' the measure $\nu$ by $2\epsilon$ to the right. Call this translated measure $\eta$. Since $f(x) \geq g(x-2\epsilon)$, it is immediate that $\eta \subseteq \mu$. Moreover, the transport cost is $D_\epsilon(\nu, \eta) = 0$. This shows that $D_\epsilon(\mu, \nu) = 0$. 
\end{proof}

\begin{lemma}\label{lemma: scrunch}
Let $\mu$ and $\nu$ be as in Theorem~\ref{thm: zero}. Assume that $\mu(\real) = \nu(\real) = U$. Suppose the following conditions hold (see Figure~\ref{fig:lemma scrunch} for an illustration): 
\begin{enumerate}
\item
Let $a, b \in \real$ be such that the support of $f$ is a subset of $[a, b]$ and the support of $g$ is a subset of $[a', b] := [a+2\epsilon, b]$. 
\item
There exists $t \in [a, b]$ such that $f(x) \geq g(x)$ for $x \in [a, t)$, and $f(x) \leq g(x)$ for $x \in (t, b]$.
\item
Let $\tilde g(x) = g(x+2\epsilon)$. Note that the support $\tilde g$ is within $[a, b-2\epsilon]$. There exists $\tilde t \in [a, b-2\epsilon]$ such that $f(x) \leq \tilde g(x)$ for $x \in [a, \tilde t)$, and $f(x) \geq \tilde g(x)$ for $x \in (\tilde t, b-2\epsilon]$. 
\end{enumerate}
Then $D_\epsilon(\mu, \nu) = 0$. A mirror image of this result also holds: $D_\epsilon(\mu, \nu) = 0$ when the support of $f$ is a subset of $[b, c+2\epsilon]$, that of $g$ is a subset of $[b, c]$, and $f(x) \leq g(x)$ for $x \in [b, t)$ and $f(x) \geq g(x)$ for $x \in [t, c+2\epsilon]$; and for $\tilde g(x) = g(x+2\epsilon)$ we have $f(x) \geq \tilde g(x)$ for $x \in [b+2\epsilon, \tilde t)$ and $f(x) \leq g(x)$ for $x \in [\tilde t, c+2\epsilon]$ .
\end{lemma}

\begin{figure}[t]
    \centering
    \includegraphics[scale=0.4]{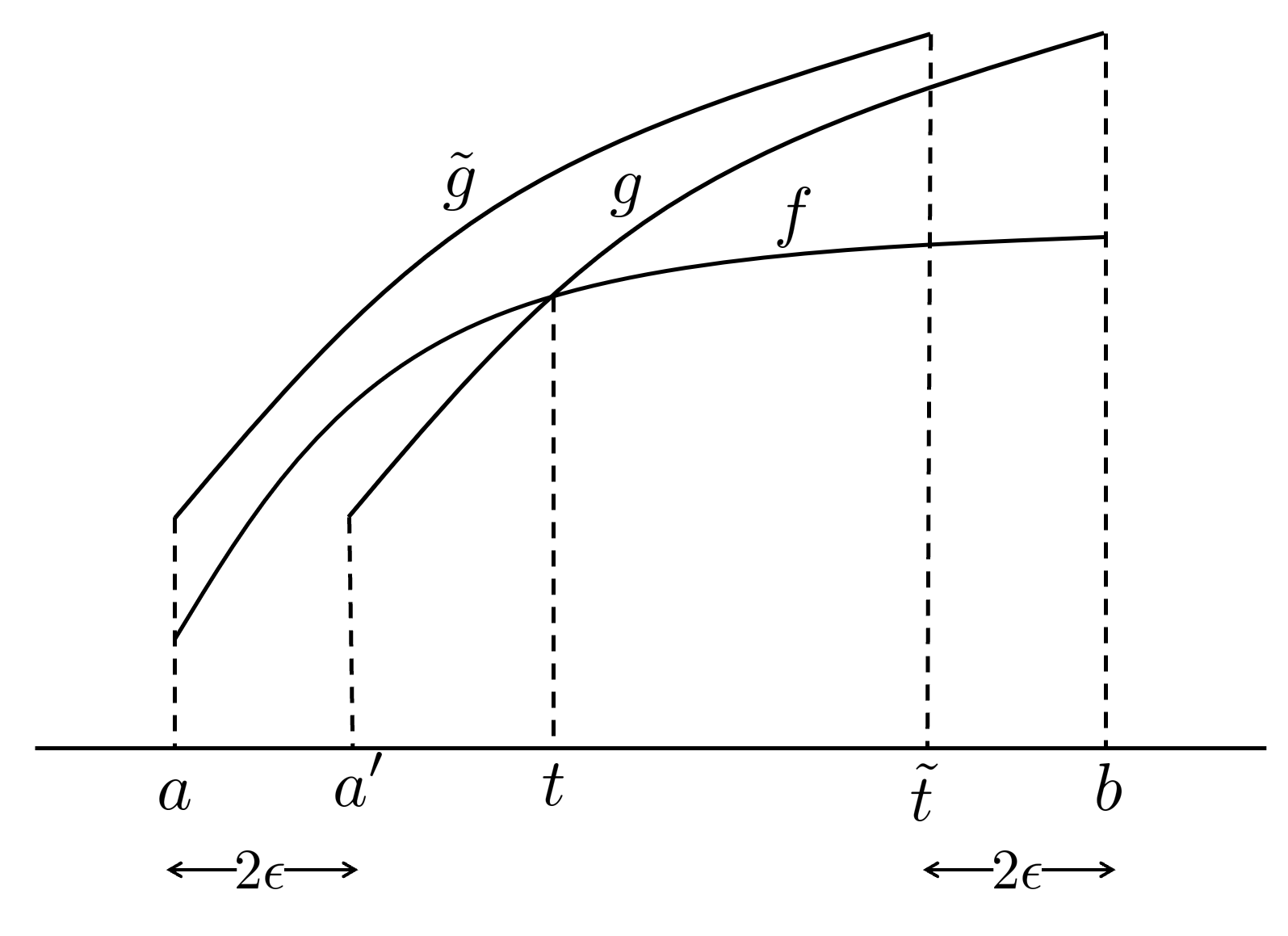}
    \caption{Figure illustrating the conditions in Lemma~\ref{lemma: scrunch}. Note that in general $\tilde{t}$ need not be equal to $b-2\epsilon$  as shown in the figure.}
    \label{fig:lemma scrunch}
\end{figure}

\begin{proof}
We first prove $F(x) \geq G(x)$. To see this, consider $H(x) = F(x) - G(x)$. Since the derivative of $H$ is $f-g$, it must be that $H$ is increasing from $[a, t)$ and decreasing from $[t, b]$. Also, we have $H(a) = H(b) = 0$, and so the function $H$ must be non-negative in $[a,b]$. Equivalently, we must have $F(x) \geq G(x)$ for $x \in \real$. We now prove $F(x) \leq G(x+2\epsilon)$. Consider $\tilde H(x) = F(x) - \tilde G(x)$. By condition $(3)$, the derivative of this function is negative from $[a, \tilde t]$ and positive from $[\tilde t, b]$. Thus, the function $\tilde H$ decreases on the interval $[a, \tilde t)$ and increases on the interval $[\tilde t, b]$. Note that since $\tilde H(a) = \tilde H(b) = 0$, the function $\tilde H$ must be non-positive in the interval $[a, b]$. Thus, we have $F(x) \leq G(x+2\epsilon)$. Applying Corollary~\ref{cor: dominate} concludes the proof.
\end{proof}

\subsection{Gaussian distributions with identical variances} \label{section: same variance}
\begin{theorem}\label{thm: same variance}
Let $p_0 = \cN(\mu_0, \sigma^2)$ and $p_1 = \cN(\mu_1, \sigma^2)$ in the metric space $(\mathbb{R}, |\cdot|)$. Assume $\mu_0 < \mu_1$ without loss of generality.  Then the following hold:
\begin{enumerate}
\item
If $\epsilon \geq \frac{\abs{\mu_0 - \mu_1}}{2}$, the optimal robust risk is $1/2$. A constant classifier achieves this risk.
\item
If $\epsilon < \frac{\abs{\mu_0 - \mu_1}}{2}$, the optimal classifier satisfies $A = \left[\frac{\mu_1+\mu_0}{2},  +\infty\right]$, where $A$ is the region where the classifier declares label $1$. The optimal risk in this case is $\int_{ \frac{\mu_1+\mu_0}{2}-\epsilon}^\infty p_0(x) dx = Q\left(\frac{\frac{\mu_1-\mu_0}{2} - \epsilon}{\sigma}\right)$.
\end{enumerate}
\end{theorem}

The lower bound of $1/2$ on the adversarial risk is trivially achieved by the constant classifier. Part (1) of the theorem states that for large enough $\epsilon$, this is the best one can do. For smaller values of $\epsilon$, the above theorem shows that the most robust classifier is the same as the MLE classifier. For larger values of $\epsilon$, the MLE classifier has a risk \emph{larger than} $1/2$; i.e., it is worse than the constant classifier. 

\begin{proof} We shall prove (1) first. Note that if $\epsilon \geq \frac{\mu_1 - \mu_0}{2}$, the transport map $T$ defined by $T(x) = x+(\mu_1 - \mu_0)$ transports $p_0$ to $p_1$. Moreover, this coupling satisfies $\abs{T(x) - x} = \mu_1- \mu_0 \leq 2\epsilon$. Thus, the optimal transport cost for this coupling is 0, and therefore so is $D_\epsilon(p_0, p_1)$. This gives the lower bound 
\begin{align*}
R_\epsilon^* \geq \frac{1}{2}.
\end{align*}
However, since the constant classifier achieves the lower bound, we conclude $R_\epsilon^* = 1/2$.

\begin{figure}[t]
\begin{center}
\includegraphics[scale = 0.45]{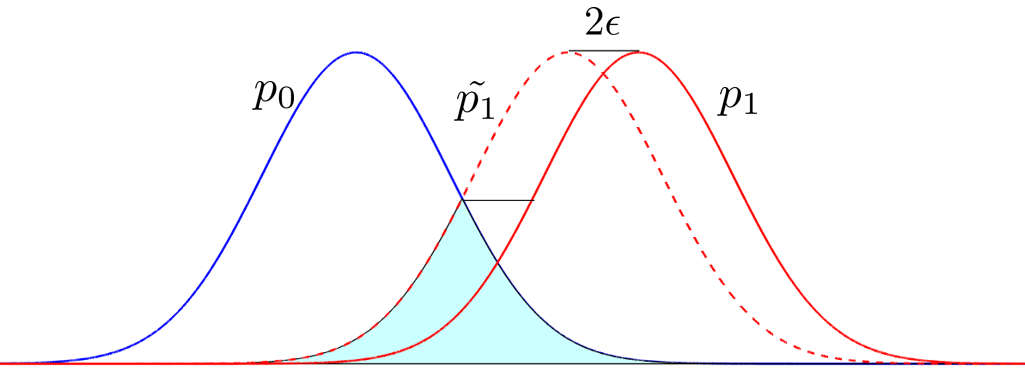}
\end{center}
\caption{Optimal coupling for two Gaussians with identical variances. The shaded region within $p_0$ is translated by $2\epsilon$ to $p_1$, whereas the remaining is mass in $p_0$ is moved at a cost of 1 per unit mass.}
\label{fig: coupling_same_var}
\end{figure}

For part (2), we consider the following strategy for transporting the mass from $p_0$ to $p_1$.
As shown in Figure~\ref{fig: coupling_same_var}, consider the distribution $\tilde p_1$ obtained by shifting $p_1$ to the left by $2\epsilon$.
That is, $\tilde p_1(x) = p_1(x+2\epsilon)$.
Define $q: \real\to \real$ as $q(x) =\min(p_0(x), \tilde p_1(x))$.
It is evident that the overlapping area between $\tilde p_1$ and $p_0$ (i.e., the area under the curve $q(x)$) maybe be translated by $2\epsilon$ to the right so that it lies entirely under the curve $p_1(x)$. More precisely, $q(x-2\epsilon)\leq p_1(x)$ for all $x\in \real$.
Hence, the area under $q(x)$ may be transported to $p_1(x)$ at $0$ cost by using the transport map $T(x) = x+2\epsilon$.
It is easily verified that the area under $q(x)$ equals $2Q\left(\frac{\frac{\mu_1-\mu_0}{2} - \epsilon}{\sigma}\right)$, and so the total cost of transporting $p_0$ to $p_1$ is at most $1-2Q\left(\frac{\frac{\mu_1-\mu_0}{2} - \epsilon}{\sigma}\right)$. Plugging this into the lower bound, we see that 
$$R_\epsilon^* \geq Q\left(\frac{\frac{\mu_1-\mu_0}{2} - \epsilon}{\sigma}\right).$$ 
Since this risk is achieved by the MLE classifier, we conclude that this is the optimal robust risk and the MLE classifier is the optimal robust classifier.
\end{proof}

Theorem~\ref{thm: same variance} can be easily extended to $d$-dimensional Gaussians with the same identity covariances. 
Our results may be summarized in the following theorem:
\begin{theorem}\label{thm: same variance d-dim}
Let $p_0 = \cN(\mu_0, \sigma^2 I_d)$ and $p_1 = \cN(\mu_1, \sigma^2 I_d)$ in the metric space $(\mathbb{R}, ||\cdot||_2)$. Then the following hold:
\begin{enumerate}
\item
If $\epsilon \geq \frac{||\mu_0 - \mu_1||_2}{2}$, the optimal robust risk is $1/2$. A constant classifier achieves this risk.
\item
If $\epsilon < \frac{||\mu_0 - \mu_1||_2}{2}$, the optimal classifier is given by the following halfspace:
\begin{align}\label{eq: A d-dim}
A = \left\{x ~:~ (\mu_1 - \mu_0) \left(x - \frac{\mu_0+\mu_1}{2} \right) \geq 0 \right\}.
\end{align}
\end{enumerate}
\end{theorem}

\paragraph{Comparison to Bhagoji et al.~\cite{BhaEtal19}:} Bhagoji et al. also explore optimal classifiers for multivariate normal distributions. In fact, they show a more general version of our Theorems~\ref{thm: same variance} and \ref{thm: same variance d-dim} by considering data distributions $\cN(\mu_0, \Sigma)$ and $\cN(\mu_1, \Sigma)$, and an adversary that perturbs within $\ell_p$-balls. 

In the following subsections, we shall generalize Theorem~\ref{thm: same variance} in a different way by considering various interesting examples of univariate distributions and identifying optimal couplings for these.

\subsection{Gaussians with arbitrary means and variances}\label{section: same mean}
We shall introduce a general coupling strategy and apply it to the special case of Gaussian random variables. Given two probability measures $\mu$ and $\nu$ on $\real$, our strategy consists of the following steps:
\begin{itemize}
\item[(1)] Partition the real line into $K \ge 1$ intervals $S_i, 1 \leq i \leq K$, and let the restriction of $\mu$ to $S_i$ be $\mu_i$.
\item[(2)] Partition the real line into $K \ge 1$ intervals $T_i, 1 \leq i \leq K$, and let the restriction of $\nu$ to $T_i$ be $\nu_i$.
\item[(3)] Transport mass from $\mu_i$ to $\nu_i$ such that $D_\epsilon(\mu_i, \nu_i) = 0$. (We shall use the definition of mass transport between measures with unequal masses from definition \ref{def: unequal}.) The transport maps used in these $K$ problems may be arbitrary; however, we shall often use versions of the monotone optimal transport map \cite{Vil03}.
\end{itemize}

Our next lemma is specific to Gaussian pdfs:
\begin{lemma}\label{lemma: normal intersection}
Let $f$ and $g$ be Gaussian pdfs corresponding to $\cN(\mu_1, \sigma_1^2)$ and $\cN(\mu_2, \sigma_2^2)$, respectively. Assume $\sigma_1^2 > \sigma_2^2$. Then the equation $f(x)-g(x) = 0$ has exactly two solutions $s_1 < \mu_2 < s_2$. 
\end{lemma}
\begin{proof}
By scaling and translating, we may set $\mu_2 = 0$ and $\sigma_2^2 = 1$. Solving $f(x) - g(x) = 0$ is equivalent to solving the quadratic equation
\begin{align*}
\frac{x^2}{2} - \frac{(x-\mu_1)^2}{2\sigma_1^2} = \log \sigma_1.
\end{align*}
Simplifying, we wish to solve
\begin{align*}x^2(\sigma_1^2 - 1) + 2\mu_1 x - (\mu_1^2 + 2\sigma_1^2 \log \sigma_1) = 0.
\end{align*}
Since $\sigma_1 > 1$, the above quadratic has two distinct roots: one negative and one positive. This proves the claim.
\end{proof}

We shall call the two points where $f$ and $g$ intersect as the left and right intersection points. 

\begin{theorem}\label{thm: same mean}
Let $\mu$ and $\nu$ be the Gaussian measures $\cN(0,\sigma_1^2)$ and $\cN(0,\sigma_2^2)$, respectively. Assume $\sigma_1^2 > \sigma_2^2$ without loss of generality. Let $m >0$ be such that $f(m+\epsilon) = g(m-\epsilon)$. Let $A = (-\infty, -m] \cup [m, +\infty)$. Then the optimal transport cost between $\mu$ and $\nu$ is given by
\begin{align*}
D_\epsilon(\mu, \nu) &= \mu(A^{\ominus \epsilon}) - \nu(A^{\oplus \epsilon}) \\ 
&= 2Q\left(\frac{m+\epsilon}{\sigma_1}\right) - 2Q\left(\frac{m-\epsilon}{\sigma_2}\right).
\end{align*}
The corresponding robust risk is
\begin{align*}
R_\epsilon^* = \frac{1 - \mu(A^{\ominus \epsilon}) + \nu(A^{\oplus \epsilon})}{2}.
\end{align*}
Moreover, if $\mu$ corresponds to hypothesis 1, the optimal robust classifier declares label $1$ on the set $A$.
\end{theorem}

\begin{proof}

We shall propose a map that transports $\mu$ to $\nu$. (See Figure~\ref{fig: gaussians_centered} for an illustration.) 
The existence of a $m>0$ such that $f(m+\epsilon) = g(m-\epsilon)$ is guaranteed by Lemma~\ref{lemma: normal intersection}.
Consider $r \in (0, m-\epsilon)$ whose value will be provided later. First, we partition $\real$ into the five regions for $\mu$ and $\nu$, as shown in Table~\ref{table: partition}. For $\mu$, these partitions are $(-\infty, -m-\epsilon]$, $(-m-\epsilon, -r]$, $(-r, +r)$, $[r, m+\epsilon)$, and $[m+\epsilon, \infty)$. Let $\mu$ restricted to these intervals be $\mu_{--}$, $\mu_{-}$, $\mu_0$, $\mu_{+},$ and $\mu_{++}$, respectively. The measure $\nu$ is also partitioned five ways, but the intervals used in this case are slightly modified to be $(-\infty, -m+\epsilon]$, $(-m+\epsilon, -r]$, $(-r, r)$, $[r, m-\epsilon)$, and $[m-\epsilon, +\infty)$. Call $\nu$ restricted to these intervals $\nu_{--}$, $\nu_{-}$, $\nu_0$, $\nu_+$, and $\nu_{++}$, respectively. 

\begin{table}[]
\begin{center}
\begin{tabular}{|c|c|}
\hline
$\mu_{--}$ &$(-\infty, -m-\epsilon]$  \\ \hline
$\mu_{-}$ &$(-m-\epsilon, -r]$  \\ \hline
$\mu_{0}$ &$(-r, +r)$  \\ \hline
$\mu_{+}$ &$[r, m+\epsilon)$  \\ \hline
$\mu_{++}$ &$[m+\epsilon, \infty)$  \\ \hline
\end{tabular}
%\end{table}
\quad
%\begin{table}[]
\begin{tabular}{|c|c|}
\hline
$\nu_{--}$ &$(-\infty, -m+\epsilon]$  \\ \hline
$\nu_{-}$ &$(-m+\epsilon, -r]$  \\ \hline
$\nu_{0}$ &$(-r, +r)$  \\ \hline
$\nu_{+}$ &$[r, m-\epsilon)$  \\ \hline
$\nu_{++}$ &$[m-\epsilon, \infty)$  \\ \hline
\end{tabular}
\end{center}
\caption{The real line is partitioned into five regions for $\mu$ and $\nu$, as shown in the table.}
\label{table: partition}
\end{table}
The transport plan from $\mu$ to $\nu$ will consist of five maps transporting $\mu_{--} \to \nu_{--}$, $\mu_- \to \nu_-$, $\mu_0 \to \nu_0$, $\mu_+ \to \nu_+$, and $\mu_{++} \to \nu_{++}$. In each case, we plan to show that $D_\epsilon(\mu_{*}, \nu_{*}) = 0$, where $*$ ranges over all possible subscripts in $\{--, -, 0, +, ++\}$. Note that these measures do not necessarily have identical masses, and thus by Definition~\ref{def: unequal}, we are transporting a quantity of mass equal to the minimum mass among the two measures. For this reason, even though the transport cost is $D_\epsilon(\mu_*, \nu_*) = 0$, it does not mean $D_\epsilon(\mu, \nu) = 0$.

Consider $\mu_{++}$ and $\nu_{++}$. We have $f(m+\epsilon) = g(m-\epsilon)$ by the choice of $m$. We argue that for any $t \geq 0$, we must have $f(m+\epsilon+t) \geq g(m-\epsilon+t)$. This is because
any two Gaussian pdfs can intersect in at most two points. By Lemma~\ref{lemma: normal intersection}, the $\epsilon$-shifted Gaussian pdfs $f(x+\epsilon)$ and $g(x-\epsilon)$ have $m$ as their right intersection point, and there are no additional points of intersection to the right of $m$. Since the tail of $f$ is heavier, it means that $f(m+\epsilon+t) \geq g(m-\epsilon + t)$ for all $t \geq 0$. By Lemma~\ref{lemma: shift}, we can now conclude $D_\epsilon(\mu_{++}, \nu_{++}) = 0$. A similar argument also shows $D_\epsilon(\mu_{--}, \nu_{--}) = 0$. 

Before we consider $\mu_{-}$ and $\nu_{-}$, we first define $r$ as follows: Pick $r>0$ such that $\mu([-m-\epsilon, -r)) = \nu([-m+\epsilon, -r))$. To see that such an $r$ must exist, consider the functions $a(t) := \mu([-m-\epsilon, t))$ and $b(t):=\nu([-m+\epsilon, t))$ as $t$ ranges over $(-m+\epsilon, 0)$. When $t = -m+\epsilon$, we have $a(t) > b(t) = 0$. When $t=0$, we have $a(t) = 1/2 - \mu_{--}(\real) < b(t) = 1/2 - \nu_{--}(\real)$.  Thus, there must exist a $t_0 \in (-m+\epsilon, 0)$ such that $a(t_0) = b(t_0)$. Pick the smallest (i.e., the leftmost) such $t_0$, and set $-r = t_0$. Call $f(\cdot)$ restricted to $[-m-\epsilon, -r)$ and $g(\cdot)$ restricted to $[-m+\epsilon, -r)$ as $f_-$ and $g_-$, respectively, and their corresponding cdfs $F_-$ and $G_-$, respectively. We claim that $\mu_-$ and $\nu_-$ satisfy all three conditions from Lemma~\ref{lemma: scrunch}. Since the supports of $f_-$ and $g_-$ are  $[-m-\epsilon, -r)$ and $[-m+\epsilon, -r)$, condition $(1)$ is immediately verified. To check condition $(2)$, we break up the interval $[-m-\epsilon, -r)$ into two parts: $[-m-\epsilon, -s)$ and $[-s, -r)$, where $s$ is such that $f(-s) = g(-s)$. Observe that $f_- \geq g_-$ on $[-m-\epsilon, -s)$, whereas $f_- \leq g_-$ on $[-s, -r)$. This shows that condition $(2)$ is satisfied. We have $g_-(-m+\epsilon) = f_-(-m-\epsilon)$. Again, using Lemma~\ref{lemma: normal intersection} the $2\epsilon$-shifted Gaussian pdf $f(x-2\epsilon)$ and $g(x)$ have $-m+\epsilon$ as their left intersection point, and the right intersection point is to the right of 0. Thus, we have $f(x-2\epsilon) \leq g(x)$ for all $x \in [-m+\epsilon, 0] \supseteq [-m+\epsilon, r).$ Using this domination, we conclude that $f_- \leq \tilde g_-$ in the interval $[-m-\epsilon, -r-2\epsilon)$ and $f_- \geq g_- = 0$ in the interval $(-r-2\epsilon, -r]$, and so condition $(3)$ is satisfied. Applying Lemma~\ref{lemma: scrunch}, we conclude $D_\epsilon(\mu_-, \nu_-) = 0$. An essentially identical argument may be used to show $D_\epsilon(\mu_+, \nu_+) = 0$. The minor difference being that $r$ is chosen to satisfy $\mu([r, m+\epsilon)) = \nu([r, m-\epsilon))$, and the mirror image of Lemma~\ref{lemma: scrunch} is applied.

Finally, consider the interval $(-r, +r)$. In this interval, $f(x) \leq g(x)$ for every point. Hence, a transport map from $\mu_0$ to $\nu_0$ is obtained by simply considering the identity function. Any remaining mass in $\mu$ is moved to $\nu$ arbitrarily, incurring a cost of at most 1 per unit mass. The total cost of transport is then upper-bounded by
\begin{align*}
D_\epsilon(\mu, \nu) &\leq 1 - \left[ \min(\mu_{--}, \nu_{--}) + \min(\mu_{-}, \nu_{-}) + \min(\mu_{0}, \nu_{0}) + \min(\mu_{+}, \nu_{+}) + \min(\mu_{++}, \nu_{++})\right]\\
&= 1 -  \left[ \nu_{--} + \mu_{-} + \mu_{0} + \mu_{+} + \nu_{++}\right]\\
&= 1 - \mu([-m-\epsilon, m+\epsilon]) - 2\nu([m-\epsilon, \infty))\\
&= \mu(A^{\ominus \epsilon}) - \nu(A^{\oplus \epsilon})\\
%&= 1 - \left(1 - 2Q\left(\frac{m+\epsilon}{\sigma_1}\right)\right) - 2Q\left(\frac{m-\epsilon}{\sigma_2}\right)\\
&= 2Q\left(\frac{m+\epsilon}{\sigma_1}\right) - 2Q\left(\frac{m-\epsilon}{\sigma_2}\right).
\end{align*}
where for brevity we have denoted $\mu_{*}(\real)$ as $\mu_*$. However, we also have
\begin{align*}
D_\epsilon(\mu, \nu) &\geq \mu(A^{\ominus \epsilon}) - \nu(A^{\oplus \epsilon}).
\end{align*}
The lower and upper bounds match and this concludes the proof. The robust risk $R_\epsilon^*$ is given by Theorem~\ref{th_01bound}. The robust risk of the classifier that declares label $1$ on the set $A$ is easily seen to be $R_\epsilon^*$.
\end{proof}

\begin{figure}[t]
    \centering
    \includegraphics[scale=0.5]{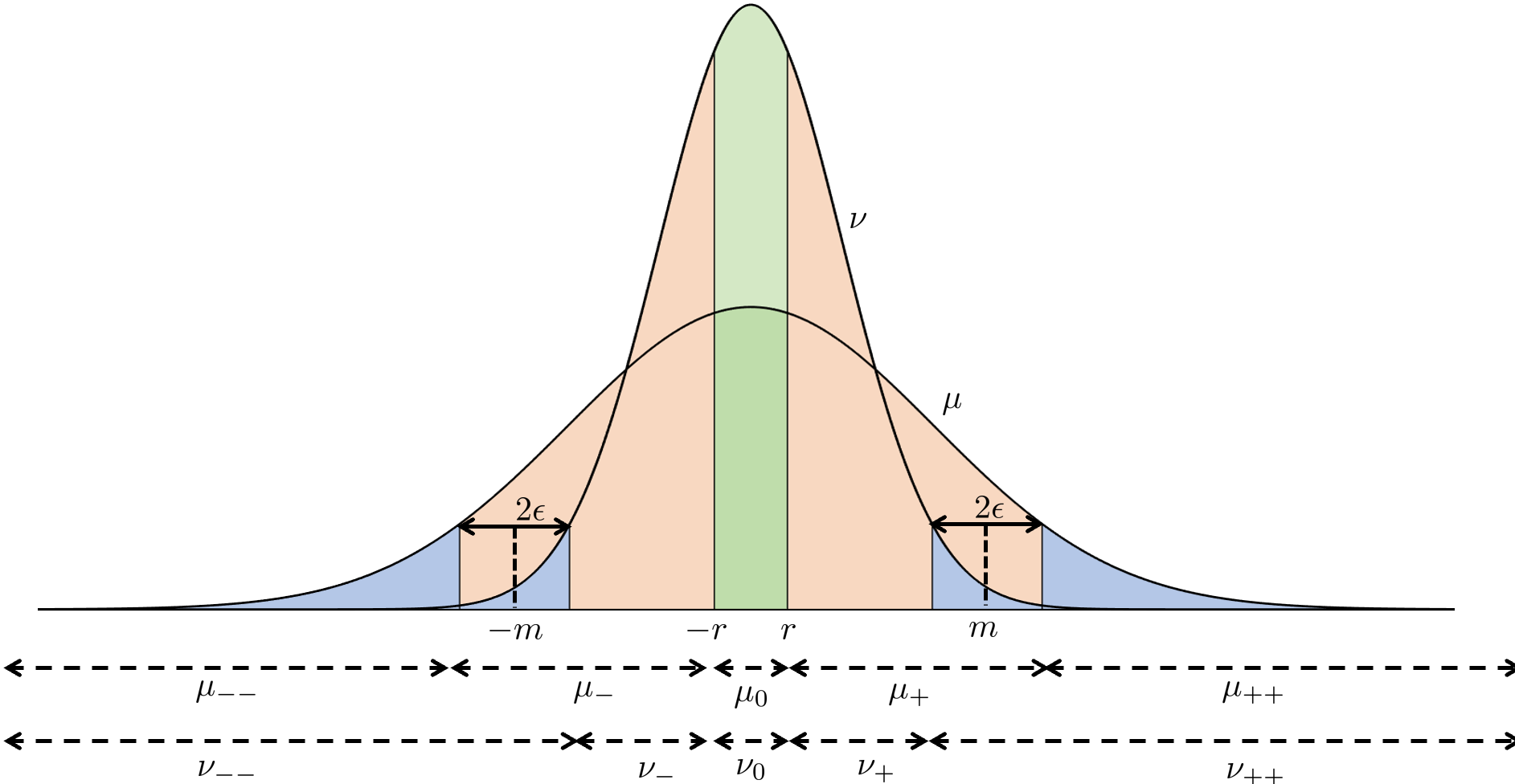}
    \caption{Optimal transport coupling for centered Gaussian distributions $\mu$ and $\nu$. As in the proof of Theorem~\ref{thm: same mean}, we divide the real line into five regions. The transport plan from $\mu$ to $\nu$ consists of five maps transporting $\mu_{--} \to \nu_{--}$ (blue regions to the left), $\mu_- \to \nu_-$ (orange regions to the left), $\mu_0 \to \nu_0$ (green regions in the middle), $\mu_+ \to \nu_+$ (orange regions to the right), and $\mu_{++} \to \nu_{++}$ (blue regions to the right).}
    \label{fig: gaussians_centered}
\end{figure}

We now extend the above proof strategy to demonstrate the optimal coupling for Gaussians with arbitrary means and arbitrary variances. Our main result is the following:
\begin{theorem}\label{thm: diff mean}
Let $\mu$ and $\nu$ be Gaussian measures $\cN(\mu_1, \sigma_1^2)$ and $\cN(\mu_2, \sigma_2^2)$ respectively. Assume $\sigma_1^2 > \sigma_2^2$ without loss of generality. Let $m_1, m_2 >0$ be such that $f(-m_1-\epsilon) = g(-m_1+\epsilon)$ and $f(m_2+\epsilon) = g(m_2-\epsilon)$. Let $A = (-\infty, -m_1] \cup [m_2, \infty)$. Then the optimal transport cost between $\mu$ and $\nu$ is given by
\begin{align*}
D_\epsilon(\mu, \nu) = \mu(A^{\ominus \epsilon}) - \nu(A^{\oplus \epsilon}).
\end{align*}
Consequently, the robust risk is given by
\begin{align*}
R_\epsilon^* = \frac{1}{2} (1 - \mu(A^{\ominus \epsilon}) + \nu(A^{\oplus \epsilon})).
\end{align*}
If $\mu$ corresponds to hypothesis 1, the optimal robust classifier declares label $1$ on the set $A$.
\end{theorem}
\begin{proof}
We first note that the existence of $m_1$ and $m_2$ in the theorem statement is guaranteed by Lemma~\ref{lemma: normal intersection}.
As in the proof of Theorem~\ref{thm: same mean}, we shall divide the real line into five regions as shown in Table~\ref{table: diff_meanvar} where we define $r_1$ and $r_2$ shortly. Using an identical strategy as in Theorem~\ref{thm: same mean}, we conclude $D_\epsilon(\mu_{--}, \nu_{--}) = D_\epsilon(\mu_{++}, \nu_{++}) = 0$. Define $r_1$ as the leftmost point where $\mu([-m_1-\epsilon, r_1)) = \nu([-m_1+\epsilon, r_1))$. Similarly, define $r_2$ to be the rightmost point such that $\mu([r_2, m_2+\epsilon)) = \nu([r_2, m_2-\epsilon))$. We shall now prove $D_\epsilon(\mu_-, \nu_-) = 0$ by using Lemma~\ref{lemma: scrunch}.  Verifying conditions $(1)$ and $(2)$ is exactly as in that of Theorem~\ref{thm: same mean}. The novel component of this proof is verifying condition $(3)$, since the domination used in the proof of Theorem~\ref{thm: same mean} does not work in this case due to the asymmetry. Consider the pdfs $f_-(x)$ and $g_-(x+2\epsilon)$. These two pdfs, being restrictions of Gaussian pdfs to suitable intervals, may only intersect in at most two points. One of these points of intersection is $-m_1-\epsilon$ by the choice of $m_1$, so there can be at most one other point of intersection in the interval $[-m_1-\epsilon, -r_1-2\epsilon]$. Note that there may be no point of intersection in this interval. However, the key observation is that in both cases, condition $(3)$ continues to be satisfied. To see this, suppose that there is a point of interaction $\tilde t$. In this case, $f_- \leq \tilde g_-$ in $[-m_1-\epsilon, \tilde t)$, and $f_- \geq g_-$ in $(\tilde t, -r_1]$. If there is no point of intersection, then $f_- \leq \tilde g_-$ in $[-m_1-\epsilon, -r_1-2\epsilon)$, and $f_- \geq g_- = 0$ in $(-r_1-2\epsilon, -r_1]$. This verifies condition $(3)$. Using Lemma~\ref{lemma: scrunch}, we conclude $D_\epsilon(\mu_-, \nu_-) = 0$. An identical approach gives $D_\epsilon(\mu_+, \nu_+) = 0$. Since $f(x) \leq g(x)$ for all points in the interval $(-r_1, r_2)$, the identity map may be used to conclude $D_\epsilon(\mu_0, \nu_0) = 0$. 

Any remaining mass in $\mu$ is moved to $\nu$ arbitrarily, incurring a cost of at most 1 per unit mass. The total cost of transport is then upper-bounded by
\begin{align*}
D_\epsilon(\mu, \nu) &\leq 1 - \left[ \min(\mu_{--}, \nu_{--}) + \min(\mu_{-}, \nu_{-}) + \min(\mu_{0}, \nu_{0}) + \min(\mu_{+}, \nu_{+}) + \min(\mu_{++}, \nu_{++})\right]\\
&= 1 -  \left[ \nu_{--} + \mu_{-} + \mu_{0} + \mu_{+} + \nu_{++}\right]\\
&= 1 - \mu([-m_1-\epsilon, m_2+\epsilon]) - \nu((-\infty, -m_1+\epsilon)) - \nu([m_2-\epsilon, \infty))\\
&= \mu(A^{\ominus \epsilon}) - \nu(A^{\oplus \epsilon}),
%&= 1 - \left(1 - 2Q\left(\frac{m+\epsilon}{\sigma_1}\right)\right) - 2Q\left(\frac{m-\epsilon}{\sigma_2}\right)\\
%&= 2Q\left(\frac{m+\epsilon}{\sigma_1}\right) - 2Q\left(\frac{m-\epsilon}{\sigma_2}\right).
\end{align*}
where for brevity we have denoted $\mu_{*}(\real)$ as $\mu_*$, where $*$ ranges over all possible subscripts in $\{--, -, 0, +, ++\}$. The rest of the proof is identical to that of Theorem~\ref{thm: same mean}. 
\begin{table}[]
\begin{center}
\begin{tabular}{|c|c|}
\hline
$\mu_{--}$ &$(-\infty, -m_1-\epsilon]$  \\ \hline
$\mu_{-}$ &$(-m_1-\epsilon, -r_1]$  \\ \hline
$\mu_{0}$ &$(-r_1, +r_2)$  \\ \hline
$\mu_{+}$ &$[r_2, m_2+\epsilon)$  \\ \hline
$\mu_{++}$ &$[m_2+\epsilon, \infty)$  \\ \hline
\end{tabular}
%\end{table}
\quad
%\begin{table}[]
\begin{tabular}{|c|c|}
\hline
$\nu_{--}$ &$(-\infty, -m_1+\epsilon]$  \\ \hline
$\nu_{-}$ &$(-m_1+\epsilon, -r_1]$  \\ \hline
$\nu_{0}$ &$(-r_1, +r_2)$  \\ \hline
$\nu_{+}$ &$[r_2, m_2-\epsilon)$  \\ \hline
$\nu_{++}$ &$[m_2-\epsilon, \infty)$  \\ \hline
\end{tabular}
\end{center}
\caption{The real line is partitioned into five regions for $\mu$ and $\nu$ as shown in the table.}
\label{table: diff_meanvar}
\end{table}
\end{proof}

\subsection{Beyond Gaussian examples}
The coupling strategy for Gaussian random variables can also be applied to other univariate examples that share some similarities with the Gaussian case. To illustrate, we describe the optimal classifier and optimal coupling for uniform distributions and triangular distributions. 

\begin{theorem}[Uniform distributions]\label{thm: uniform coupling}
Let $\mu$ and $\nu$ be uniform measures on closed intervals $I$ and $J$ respectively. Without loss of generality, we assume $|I| \leq |J|$. Then the optimal robust risk is $\nu(I^{2\epsilon})$ and the optimal classifier is given by $A = I^\epsilon$.
\end{theorem}
\begin{proof}
See Appendix~\ref{app: uniform coupling proof}
\end{proof}

In the following, we present the optimal adversarial risk and optimal classifier for symmetric triangular distributions. For $\delta>0$, we use $\Delta(m, \delta)$ to denote a triangular distribution with support $[m-\delta, m+\delta]$ and mode at $m$. The pdf of such a distribution is given by the function $f(x) = \frac{1}{\delta}\max\left\{ 1 - \frac{|x-m|}{\delta}, 0 \right\}$.

The next lemma is similar to Lemma~\ref{lemma: normal intersection}, but is specific to symmetric triangular distributions.
\begin{lemma}\label{lemma: triangular intersection}
Let $\mu$ and $\nu$ correspond to the triangular distributions $\Delta(m_1, \delta_1)$ and $\Delta(m_2, \delta_2)$ with pdfs $f$ and $g$ respectively. Assume $\delta_1 < \delta_2$. 
Then,
\begin{enumerate}
	\item If $|m_1-m_2| > \delta_2 + \delta_1$, then the equation $f(x)-g(x)=0$ has no solutions on the supports of $\mu$ or $\nu$.
	\item If $\delta_2 - \delta_1 <|m_1-m_2| \leq \delta_2 + \delta_1$, then the equation $f(x)-g(x)=0$ has exactly one solution $u$ on the support of $\mu$. Further, $u\geq m_1$ if and only if $m_1\leq m_2$.
	\item If $|m_1-m_2| \leq \delta_2 - \delta_1$, then the equation $f(x)-g(x)=0$ has exactly two solutions $l\in [m_1-\delta_1, m_1]$ and $r\in [m_1, m_1+\delta_1]$ on the support of $\mu$.
\end{enumerate}

\end{lemma}
\begin{proof}

We may assume that $m_1\leq m_2$, as case of $m_1\geq m_2$ follows by symmetry.

Suppose $|m_1-m_2| > \delta_2 + \delta_1$. Then $m_1+\delta_1<m_2-\delta_2$. Hence, the supports of $\mu$ and $\nu$ are disjoint and the result follows trivially.

Suppose $\delta_2 - \delta_1 <|m_1-m_2| \leq \delta_2 + \delta_1$. Then, $m_1+\delta_1\in [m_2-\delta_2, m_2+\delta_2]$ and $m_1-\delta_1\notin [m_2-\delta_2, m_2+\delta_2]$. Hence, the only solution $u$ to $f(x)-g(x)=0$ occurs at the intersection of the graph of $g(x)$ with the line segment joining the points $(m_1, 1/\delta_1)$ and $(m_1+\delta_1, 0)$. Clearly, $u\geq m_1$.

Suppose $|m_1-m_2| \leq \delta_2 - \delta_1$.
Then, $m_1-\delta_1\geq m_2-\delta_2$ and $m_1+\delta_1\leq m_2-\delta_2$. Hence, $[m_1-\delta_1, m_1+\delta_1]\subset [m_2-\delta_2, m_2+\delta_2]$. 
It follows that $f(m_1-\delta_1) - g(m_1-\delta_1)<0$, $f(m_1) - g(m_1)>0$ and $f(m_1+\delta_1) - g(m_1+\delta_1)<0$.
Since $f(x)-g(x)$ is a continuous function, there must be $l\in [m_1-\delta_1, m_1]$ and $r\in [m_1, m_1+\delta_1]$ such that $f(l)-g(l)=0$ and $f(r)-g(r)=0$. Moreover, $f(x)>g(x)$ for $x\in (l,r)$ and $f(x)<g(x)$ for $x\in (m_2-\delta_2, l)\cup (r, m_2+\delta_2)$.
Hence, $l$ and $r$ are the only solutions to $f(x)-g(x)=0$ on the support of $\mu$.

\end{proof}

\begin{theorem}[Triangular distributions]\label{thm: triangle coupling}
Let $\mu$ and $\nu$ correspond to the triangular distributions $\Delta(m_1, \delta_1)$ and $\Delta(m_2, \delta_2)$ with pdfs $f$ and $g$ respectively. Without loss of generality, assume $\delta_1 < \delta_2$ and $m_1<m_2$ (the case of $m_1>m_2$ follows from symmetry).  Let $2\epsilon \in (0,  \min(2\delta_1, \delta_2-\delta_1))$. Let $l = \sup\{x\leq m_1: f(x+\epsilon) = g(x-\epsilon)\}$  and $r = \inf\{x\geq m_1: f(x-\epsilon) = g(x+\epsilon)\}$. Let $A$ be the set defined as follows.
\begin{enumerate}
	\item If $m_2-m_1 \ge  \delta_2+\delta_1+2\epsilon$, then $A = (-\infty, m_1+\delta_1+\epsilon]$.
	\item If $\delta_2-\delta_1-2\epsilon \le m_2-m_1 <  \delta_2+\delta_1+2\epsilon$,  then $A = (-\infty, r]$.
	\item If $m_2-m_1 <  \delta_2-\delta_1-2\epsilon$,  then $A = [l, r]$.
\end{enumerate}
Then $D_\epsilon(\mu, \nu) = \mu(A^{\ominus \epsilon}) -\nu(A^{\oplus \epsilon})$, and the robust risk is 
\begin{align*}
R_\epsilon^* = \frac{1-\mu(A^{\ominus \epsilon}) + \nu(A^{\oplus \epsilon})}{2},
\end{align*}
and if $\mu$ corresponds to hypothesis 1, then the optimal robust classifier declares label $1$ on $A$.
\end{theorem}
\begin{proof}
See Appendix~\ref{app: triangle coupling proof}
\end{proof}

\section{Adversarial risk for continuous loss functions}\label{gen loss}

It is natural to ask if the results for $0$-$1$ loss may be extended to continuous losses. In this section, we present adversarial risk bounds in regression-like settings with continuous losses and investigate Questions~\ref{q_1} and \ref{q_2} in light of these bounds. 
%We recall that the adversary's perturbation ball around $x$ is according to a metric $d(\cdot, \cdot)$, which need not be the same as norm of the Hilbert space $\cX$. 

\subsection{Optimal adversarial risk}\label{sec: opt risk gen loss}

In this section, we prove lower and upper bounds on the optimal adversarial risk for distribution perturbing adversaries with budget $\epsilon \ge 0$. To prove lower bounds, we consider the $W_\infty$-distribution perturbing adversary with budget $\epsilon$, since this bound is valid for all $W_p$-distribution perturbing adversaries. Similarly, we prove upper bounds for the $W_1$-distribution perturbing adversary with budget $\epsilon$.  
%In this section, we prove a lower bound on optimal adversarial risk for for the data perturbing adversary and an upper bound optimal adversarial risk for for the distribution perturbing adversary. As noted earlier, the lower bound is also valid for distribution perturbing adversary and the upper bound is also valid for data perturbing adversary.

\subsubsection*{A trivial lower bound}\label{sec_triv_bound}
We start by presenting a trivial lower bound on the optimal adversarial risk. We shall assume that for all $\epsilon \ge 0$, the optimal hypothesis $w_\epsilon^*$ exists to  simplify presentation. The proofs can be easily modified by considering sequences of hypothesis such that $\liminf_i R_\epsilon^*(w_i) = R_\epsilon^*$ in case $w_\epsilon^*$ does not exist. %To derive the bound, we use the fact that the adversarial loss $R_\epsilon(\ell, w)$ for any classifier (i.e., for any $w$) is lower-bounded by the standard loss $R_0(\ell, w)$ for that classifier.

\begin{theorem}\label{th_trivial}
The optimal adversarial risk is at least as large as the optimal standard risk, that is, $R^*_\epsilon \geq R^*_0$.
\end{theorem}
\begin{proof}
We have the sequence of inequalities:
\begin{align*}
    R^*_\epsilon
    = R_\epsilon(\ell, w^*_\epsilon)
    \geq R_0(\ell, w^*_\epsilon)
    \geq R_0(\ell, w^*_0)
    = R^*_0.
\end{align*}
The first inequality holds because $R_\epsilon(\ell, w)$ is a non-decreasing function of $\epsilon$ for any fixed $\ell$ and $w$.
The second inequality follows from the fact that the adversarially optimal classifier $w^*_\epsilon$ is sub-optimal for minimizing the standard risk.
\end{proof}

Note that the bound in Theorem \ref{th_trivial} does not depend on the strength of the adversary $\epsilon$, and hence it may not be very tight for large $\epsilon$. In what follows, we show tighter lower bounds for $R^*_\epsilon$ that depend on $\epsilon$.

For the lower bound, we consider loss functions that are convex with respect to the input $x$, as defined below.
\begin{definition}[Convex loss function]
We say that the loss function $\ell:\cX \times \cY \times\cW\to\mathbb{R}^+$ is convex with respect to the input if it satisfies the following condition.
\begin{align}\label{eq_convex}
    \ell((x', y), w) - \ell((x, y), w)
    &\geq
    \langle\nabla_x \ell((x,y), w), x'-x\rangle.
\end{align}
\end{definition}

\begin{theorem}\label{th_bound_generalloss}
The adversarial risk for a loss function satisfying \eqref{eq_convex} is bounded as follows.
\begin{align}
    R^*_\epsilon
    \geq
    R^*_0 + 
    \inf_{w \in \cW} \E_z\left[\sup_{d(x,x')\leq\epsilon}
    \langle \nabla_x\ell((x,y),w), x'-x \rangle
    \right].
\end{align}
\end{theorem}
\begin{remark*}
The lower bound holds for any $p$-Wasserstein distribution perturbing adversary with budget $\epsilon$.
\end{remark*}

Note that adversary's metric $d(\cdot, \cdot)$ may not be the same as the norm on the Hilbert space $\cX$. In the special case $d$ corresponds to the norm $\lVert \cdot\rVert_\text{adv}$, we can tighten the result of Theorem~\ref{th_bound_generalloss} as follows.
 
\begin{corollary}\label{cor_dualnorm}
In the setting of Theorem~\ref{th_bound_generalloss}, if $d(x,x')=\lVert x-x'\rVert_\text{adv}$ for $x,x'\in\cX$, then the following bound holds:
\begin{align}
        R^*_\epsilon
    \geq
    R^*_0
    + \epsilon\inf_{w\in\mathcal{W}}\E_{z}[\lVert \nabla_x\ell((x,y),w) \rVert_\text{adv*}],
\end{align}
where $\lVert \cdot\rVert_\text{adv*}$ is the dual norm of $\lVert \cdot\rVert_\text{adv}$.
\end{corollary}

\begin{proof}[Proof of Theorem~\ref{th_bound_generalloss}]

Recall the notation $\cZ = \cX \times \cY$, and $z = (x,y)$. Since $w^*_\epsilon$ is sub-optimal for minimizing standard risk, we have 
\begin{align*}
    \E_z[\ell((x, y),w^*_\epsilon)]\geq \E_z[\ell((x, y),w^*_0)].
\end{align*}
Hence,
\begin{align*}
    R^*_\epsilon - R^*_0
    &= \E_z\left[\sup_{d(x,x')\leq\epsilon}\ell((x', y),w^*_\epsilon)\right]
    - \E_z[\ell((x, y),w^*_0)]\\
    &\geq \E_z\left[\sup_{d(x,x')\leq\epsilon}\ell((x', y),w^*_\epsilon)\right]
    - \E_z[\ell((x, y),w^*_\epsilon)]\\    
    &= \E_z\left[\sup_{d(x,x')\leq\epsilon}\ell((x', y),w^*_\epsilon)
    - \ell((x, y),w^*_\epsilon)\right]\\
    &\geq 
    \E_z\left[\sup_{d(x,x')\leq\epsilon}
    \langle \nabla_x\ell((x,y),w^*_\epsilon), x'-x \rangle
    \right],\\
    &\geq \inf_{w \in \cW} \E_z\left[\sup_{d(x,x')\leq\epsilon}
    \langle \nabla_x\ell((x,y),w^*_\epsilon), x'-x \rangle
    \right].
\end{align*}
%where we used the convexity of the loss function with respect to $\bx$ in the last inequality. Plugging in $x' = x + \epsilon\frac{\nabla_x\ell((x,y),w^*_\epsilon)}{\lVert \nabla_x\ell((x,y),w^*_\epsilon) \rVert}$, we get
%\begin{align*}
%R^*_\epsilon - R^*_0 
%&\geq
%\epsilon\E_z\left[  \lVert \nabla_x\ell((x,y),w^*_\epsilon)\rVert \right]\\
%&\geq 
%\epsilon\inf_{w\in\cW}\E_z\left[  \lVert \nabla_x\ell((x,y),w)\rVert \right].
%\end{align*}
\end{proof}

\begin{proof}[Proof of Corollary~\ref{cor_dualnorm}]

From the proof of Theorem~\ref{th_bound_generalloss}, we have 
\begin{align*}
    R^*_\epsilon - R^*_0
&\geq \E_z\left[\sup_{d(x,x')\leq\epsilon}\langle\nabla_x\ell((x', y),w^*_\epsilon), x'-x \rangle
    \right]. 
\end{align*}
Under the condition that $d(x,x') = \|x-x'\|_\text{adv}$,
\begin{align*}
    \sup_{d(x,x')\leq\epsilon}\langle\nabla_x\ell((x', y),w^*_\epsilon), x'-x \rangle
    &=
    \sup_{\|\delta\|_\text{adv}\leq\epsilon}\langle\nabla_x\ell((x', y),w^*_\epsilon), \delta \rangle\\
    &= \epsilon \|\nabla_x\ell((x', y),w^*_\epsilon)\|_\text{adv*}.
\end{align*}
\end{proof}

Next, we prove an upper bound for the adversarial risk for a $W_1$-distribution perturbing adversary. As noted earlier, this upper bound also holds for a $W_p$-distribution perturbing adversary of the same budget, where $1\leq p\leq \infty$. We make the following assumption on the loss function:

\begin{definition}[$L_w$-Lipschitz loss function]
We say that the loss function $\ell:\cZ\times\cW\to\mathbb{R}^+$ is $L_w$-Lipschitz with respect to the input if it satisfies the following condition.
\begin{align}\label{eq_lipschitz}
    \lvert\ell((x',y),w) - \ell((x,y), w)\rvert
    &\leq
    L_w\lVert x'-x \rVert.
\end{align}
\end{definition}

\begin{theorem}\label{th_lowerbound_risk}
The adversarial risk for a $W_1$-distribution perturbing adversary with budget $\epsilon$ satisfies $\widehat R^{1,*}_\epsilon \leq R^*_0 + \epsilon L_{w_0^*}.$ %Naturally, we also have $R^*_\epsilon \leq R^*_0 +  \epsilon L_{w_0^*}.$ 
\end{theorem}

The proof of this result uses an optimal transport idea from~\cite{TovJog18}.

\begin{proof}[Proof of 
Theorem~\ref{th_lowerbound_risk}]
 Suppose that the infimum for $\widehat R^{1,*}_\epsilon$ in equation \eqref{eq: optimal distr. robust risk} is attained at $\widehat w^*_\epsilon$ and the supremum for $\widehat R^1_\epsilon(\ell, \widehat w^*_\epsilon)$ in equation \eqref{eq: distr_loss1_formal} is attained for $\gamma^* \in \Gamma^1_\epsilon$.
For $y\in \cY$, recall that $\rho^{\gamma^*}_{x'|y} \in \cP(\cX)$ denotes the distribution of the perturbed data point $x'\in \cX$. Let $\pi_y \in \Pi(\rho_{x|y}, \rho^{\gamma^*}_{x'|y})$ be such that $W_1(\rho_{x|y}, \rho^{\gamma^*}_{x'|y}) = \E_{(x,x')\sim \pi_y} d(x,x')$.
  Then
\begin{align*}
    \widehat R^{1,*}_\epsilon - R^*_0
    &= \E_{(x',y)\sim \rho_y \rho^{\gamma^*}_{x'|y}} \ell((x',y), \widehat w^*_\epsilon) - \E_{(x,y)\sim \rho_y \rho_{x|y}} \ell((x,y),w^*_0)\\
    &\stackrel{(a)}\leq \E_{(x',y)\sim \rho_y \rho^{\gamma^*}_{x'|y}} \ell((x',y), \widehat w^*_0) - \E_{(x,y)\sim \rho_y \rho_{x|y}} \ell((x,y),w^*_0)\\
    &\stackrel{(b)}= \E_y \E_{(x,x')\sim \pi_y} [\ell((x',y), \widehat w^*_0) - \ell((x,y), \widehat w^*_0)]\\
    &\stackrel{(c)}\leq \E_y \E_{(x,x')\sim \pi_y} d(x,x')\cdot L_{w_0^*}\\
    &\stackrel{(d)}\leq \epsilon L_{w_0^*}.
\end{align*}
Here, (a) follows from the definition of $\widehat w^*_\epsilon$, (b) follows from linearity of expectation since $\pi_y$ is a coupling of $(x,x')$ that preserves the marginals, (c) follows from the Lipschitz assumption and (d) follows from the fact that $\gamma^* \in \Gamma^1_\epsilon$.
\end{proof}

\subsection{Optimal adversarial classifier}\label{sec: opt classifier}

In Sections~\ref{sec: opt risk 01 loss} and \ref{sec: opt risk gen loss}, we looked at Question~\ref{q_1} and showed that the adversarial risk can be strictly lower-bounded as a function of adversarial budget $\epsilon$. In this section, we tackle Question~\ref{q_2} and analyze how $w^*_\epsilon$ or $\widehat w^*_\epsilon$ may deviate from $w^*_0$. For the case of $0$-$1$ loss, the optimal classifier can change drastically even with small change in the adversarial budget $\epsilon$. For instance, consider the setting of Theorem~\ref{thm: same variance}. When $\epsilon$ changes from being less than $\frac{\abs{\mu_0 - \mu_1}}{2}$ to greater than $\frac{\abs{\mu_0 - \mu_1}}{2}$, the optimal classifier changes from a halfspace to a constant classifier. Studying the $0$-$1$ loss is hard because closed sets
are not parametrized easily. Hence we focus on the
case of convex loss functions--where convexity is with respect to $w$---to derive bounds in this section. Deriving bounds without strong convexity assumptions appears challenging. To see this, observe that there may be multiple global optima $w^*_0$ when $\epsilon = 0$. The optimal hypothesis can jump from one global optimal to a different one---possibly far away---even without any adversary.

Since our proof technique uses the upper and lower bounds for adversarial losses obtained in Section~\ref{sec: opt risk gen loss}, the bounds for deviation of $w^*_\epsilon$ and $\widehat w^*_\epsilon$ are identical. Now, we prove a theorem on how much the optimal classifier can change in the presence of an adversary.

\begin{theorem}\label{th_opt_clsfr_deviation}
For a loss function $\ell$ that satisfies \eqref{eq_lipschitz}, and is $\lambda$-strongly convex with respect to $w$, the following result holds:
\begin{align}\label{eq_opt_clsfr_deviation}
    \lVert w^*_\epsilon - w^*_0 \rVert 
    \leq 
    \sqrt{
    \frac{2\epsilon L_{w_0^*}}
    {\lambda}
    }.
\end{align}
\end{theorem}

\begin{proof}[Proof of Theorem~\ref{th_opt_clsfr_deviation}]

We have the following series of inequalities.
\begin{align*}
    \epsilon L_{w_0^*}
    &\stackrel{(a)}\geq R^*_\epsilon - R^*_0\\
    &\stackrel{(b)}\geq R_0(\ell, w^*_\epsilon) - R_0(\ell, w^*_0)\\
    &\stackrel{(c)}\geq  \frac{\lambda}{2}(\nabla^2_w R_0(\ell, w^*_0)) \lVert w^*_\epsilon - w^*_0 \rVert^2.
\end{align*}
Here, (a) follows from Theorem~\ref{th_lowerbound_risk}, (b) follows from the fact that $w^*_\epsilon$ is sub-optimal for minimizing $R_0(\ell, w)$, and (c) follows from the $\lambda$-strong convexity of $\ell$ with respect to $w$.
\end{proof}
The above theorem shows that larger values of $\lambda$ prevent the adversary from changing the hypothesis drastically. If the loss function is merely convex but not strongly convex, adding a quadratic penalty $\frac{\lambda}{2}\|w\|^2$ to the loss function will ensure strong convexity.

\section{Experiments}\label{sec: expts}

In this section, we present lower bounds on the optimal adversarial risk for empirical distributions derived from several real world datasets. 

\begin{figure}[]
     \centering
     \begin{subfigure}[bt]{0.35\textwidth}
         \centering
         \includegraphics[width=\textwidth]{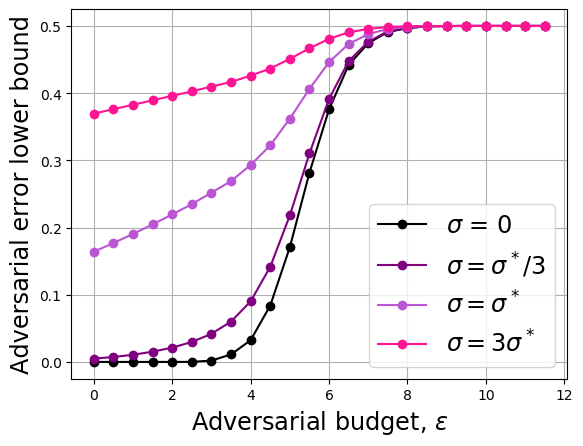}
         \caption{CIFAR10 $\ell_2$}
         \label{fig:cifar_l2}
     \end{subfigure}
     ~
     \begin{subfigure}[bt]{0.35\textwidth}
         \centering
         \includegraphics[width=\textwidth]{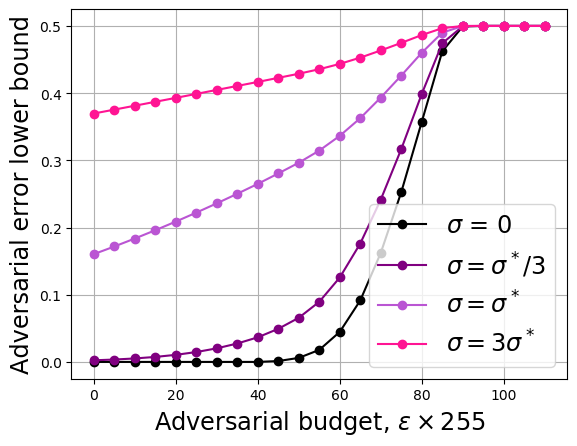}
         \caption{CIFAR10 $\ell_\infty$}
         \label{fig:cifar_linf}
     \end{subfigure}
     
     \begin{subfigure}[bt]{0.35\textwidth}
         \centering
         \includegraphics[width=\textwidth]{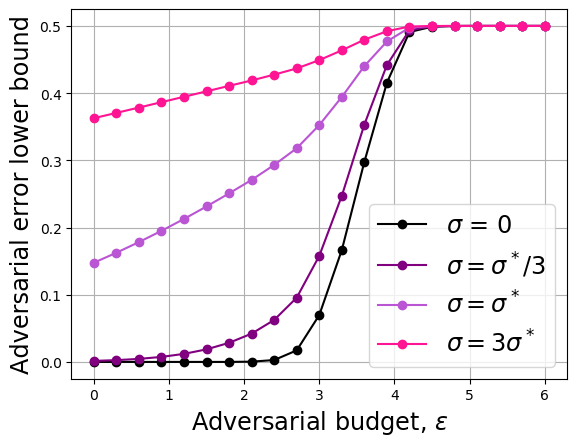}
         \caption{MNIST $\ell_2$}
         \label{fig:mnist_l2}
     \end{subfigure}
     ~
     \begin{subfigure}[bt]{0.35\textwidth}
         \centering
         \includegraphics[width=\textwidth]{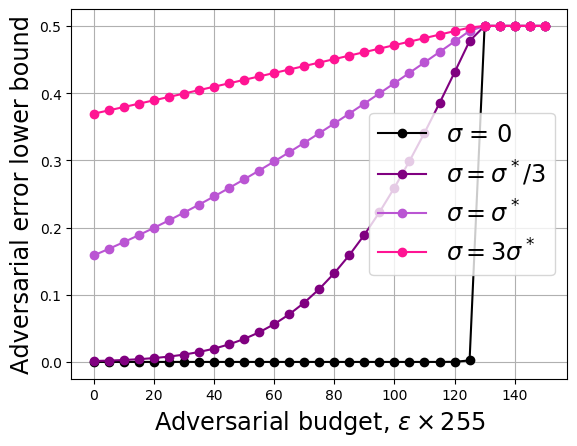}
         \caption{MNIST $\ell_\infty$}
         \label{fig:mnist_linf}
     \end{subfigure}    
     
     \begin{subfigure}[bt]{0.35\textwidth}
         \centering
         \includegraphics[width=\textwidth]{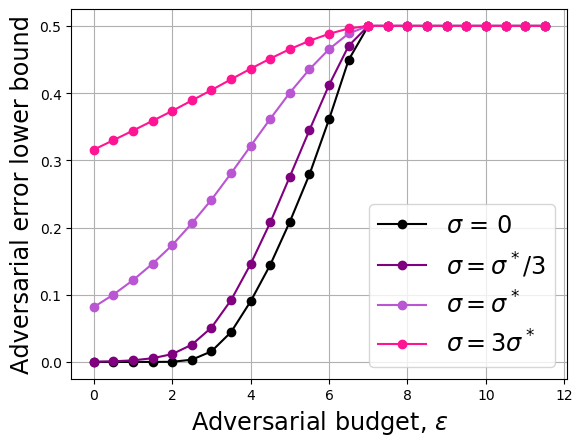}
         \caption{Fashion-MNIST $\ell_2$}
         \label{fig:fmnist_l2}
     \end{subfigure}
     ~
     \begin{subfigure}[bt]{0.35\textwidth}
         \centering
         \includegraphics[width=\textwidth]{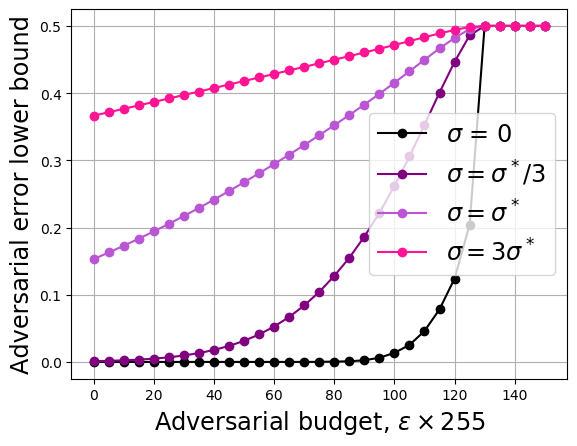}
         \caption{Fashion-MNIST $\ell_\infty$}
         \label{fig:fmnist_linf}
     \end{subfigure}      
     
     \begin{subfigure}[bt]{0.35\textwidth}
         \centering
         \includegraphics[width=\textwidth]{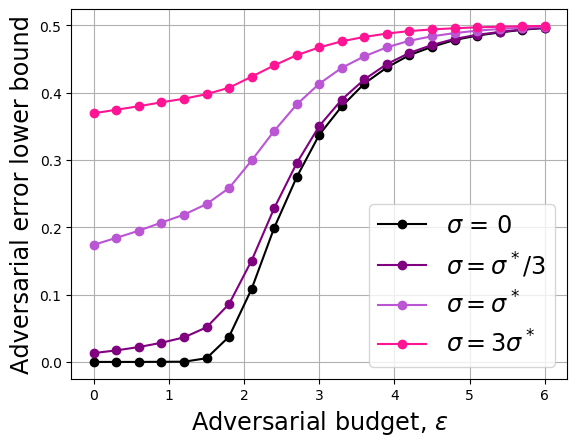}
         \caption{SVHN $\ell_2$}
         \label{fig:svhn_l2}
     \end{subfigure}
     ~
     \begin{subfigure}[bt]{0.35\textwidth}
         \centering
         \includegraphics[width=\textwidth]{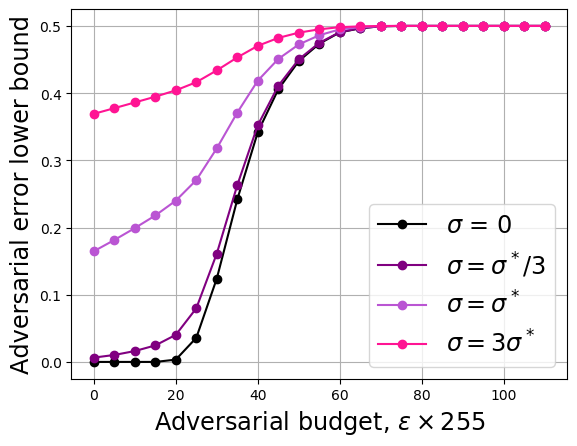}
         \caption{SVHN $\ell_\infty$}
         \label{fig:svhn_linf}
     \end{subfigure}        
     \caption{Lower bounds on adversarial risk computed using Theorem~\ref{th_01bound}. The curves with $\sigma=0$ gives the exact optimal risk for empirical distributions, while the other cuvers give lower bounds on the optimal risk for Gaussian mixtures based on the empirical distributions using the coupling in Theorem~\ref{thm: same variance d-dim}.}
     \label{fig: datasets}
\end{figure}

For the case of empirical distributions, the computation of the optimal transport cost in \eqref{eq_epsilon_Wasserstein} can be formulated as a linear program and solved efficiently. Moreover, when the number of data points in the two empirical distributions is the same, the problem of finding the optimal coupling between the two distributions is reduced to an assignment problem (see Proposition 2.11 in \cite{Pey19}), wherein the task is to optimally match each data point from the first distribution to a distinct data point from the second distribution. Using this methodology, we evaluate the optimal risk for $\ell_2$ and $\ell_\infty$ adversaries for classes $3$ and $5$ in CIFAR10, MNIST, Fashion-MNIST and SVHN datasets.  The results for other pairs of classes are very similar, and are therefore omitted for brevity.
For MNIST, Fashion-MNIST and SVHN datasets, we evaluate the optimal adversarial risk given in Theorem~\ref{th_01bound} by randomly sampling $5000$ data points from each class. The results are showing in Figure~\ref{fig: datasets} with the legend $\sigma=0$.

Since a major fraction of the data points in the empirical distributions are well-separated in $\ell_2$ and $\ell_\infty$ metrics, the optimal risk bound remains $0$ even for high $\epsilon$. For instance, for CIFAR10 dataset, the optimal risk remains $0$
for $\epsilon$ as high as $40/255$ for $\ell_\infty$. Similar results were also obtained in Bhagoji et al.~\cite{BhaEtal19}.
However, the optimal risk bounds for the true distributions may not be $0$ for high $\epsilon$, as it is unreasonable to expect a perfectly robust optimal classifier under very strong adversarial perturbations.
In addition, a common technique while training for a classifier is to augment the dataset with Gaussian perturbed samples for robustness and generalization \cite{HolKoi92, GooEtal16}. 
Motivated by this, 
we also compute optimal risk lower bounds on  Gaussian mixture distribution with the data points as the centers with scaled identity covariances. 
$\sigma = 0$ corresponds to the empirical distribution of the data points from the two classes. As $\sigma$ increases, the overlap in the probability mass between the two classes increases. This allows for the cost of optimal coupling that achieves $D_\epsilon$ to decrease, thus leading to a higher, possibly non-trivial bound for $R^*_\epsilon$.

To compute the optimal risk lower bound for Gaussian mixture, we use a coupling between the mixture distributions in two steps. In the first step, we solve for the optimal coupling that gives the exact optimal risk for the empirical distributions. This gives a pairwise matching of data points between the two empirical distributions. In the second step, we use the optimal coupling for multidimensional Gaussians from Theorem~\ref{thm: same variance d-dim} to transport the mass in the Gaussians within each pair. Overall, this transport map gives an upper bound on the $D_\epsilon$ optimal transport cost between the two mixture distributions. Using this, we obtain the lower bounds on adversarial risk shown in Figure~\ref{fig: datasets}.

Figure~\ref{fig: datasets} shows the lower bounds for various values of the variance $\sigma$ used for the Gaussian mixture, where $\sigma^*$ is half of the mean distance between data points from the two distributions. 
As explained previously, we see in Figure~\ref{fig: datasets} that the lower bound curves for higher values of $\sigma$ are above those for lower values. For instance, the optimal risk for CIFAR10 dataset under $\ell_2$ perturbation with $\epsilon=3$ is $0.25$  for $\sigma = \sigma^*$. That is, the adversarial error rate for CIFAR10 with $\epsilon = 3$ for any algorithm cannot be less than $0.25$ even when trained with Gaussian data augmentation (with $\sigma = \sigma^*$). 
In comparison, the lower bound obtained in Bhagoji et al. \citep{BhaEtal19} (which is equivalent to the case of $\sigma=0$) is $0$ for $\epsilon = 3$.
Computation of  non-trivial lower bounds for higher values of $\epsilon$ on adversarial error rate as in Figure~\ref{fig: datasets} is made possible by our analysis on the optimal coupling to achieve $D_\epsilon$ between multivariate Gaussians in section~\ref{section: same variance}.

\section{Discussion}

In this paper, we have analyzed two notions of {\em adversarial risk}: one resulting from a distribution perturbing adversary ($\widehat{R}^*_\epsilon$) and the other from a data perturbing adversary ($R^*_\epsilon$). We have introduced the $D_\epsilon$ optimal transport distance between probability distributions. Through an application of duality in the optimal transport cost formulation (via Strassen's theorem), we have shown that  $D_\epsilon$ completely characterizes the optimal adversarial risk $R^*_\epsilon$ for the case of binary classification under $0$-$1$ loss function. For general loss functions, we give lower bounds on $R^*_\epsilon$ and upper bounds on $\widehat{R}^*_\epsilon$ in terms of the Lipschitz and strong convexity parameters of the loss function.

Our analysis raises several interesting questions: How big is the gap between $\widehat{R}^*_\epsilon$ and $R^*_\epsilon$ for different kinds of loss functions? Is it possible to directly lower bound $\widehat{R}^*_\epsilon$ without appealing to its dependence on $R^*_\epsilon$? Does there exist an optimal transport distance akin to $D_\epsilon$ that characterizes $\widehat{R}^*_\epsilon$? As evidenced by experiments, our bounds for general loss functions are not particularly tight. Furthermore, we need fairly strong assumptions such as convexity and Lipschitz property for the loss function to state these bounds. It would be interesting to study if these conditions may be relaxed and if tighter bounds could be obtained.

In analysing the adversarial risk for $0$-$1$ loss functions, we give a novel coupling strategy based on monotone mappings that solves the $D_\epsilon$ optimal transport problem for symmetric unimodal distributions like Gaussian, triangular, and uniform distributions. Employing the duality in the optimal transport, we also obtain the adversarially optimal classifier under these settings. Our coupling analysis calls for an interesting open question: Is there a general coupling strategy, akin to the maximal coupling strategy to achieve the total variation transport cost, that works for a broader class of distributions? If yes, this gives us a handle on analyzing the nature of optimal decision boundaries in the adversarial setting. Optimal transport between measures with unequal mass has received attention in recent work~\cite{ChiEtal18}. We plan to investigate if the version of transport from Definition~\ref{def: unequal} is useful in other contexts, and whether computational methods as in~\cite{Pey19} may be used to compute it in practice.

Our analysis for $0$-$1$ loss reveals how the optimal risk smoothly changes from Bayes risk as the data perturbing budget $\epsilon$ is increased. Somewhat more surprisingly, our analysis shows that in some cases, the optimal classifier can change abruptly in the presence of an adversary even for small changes in $\epsilon$.
It remains to be seen if these observations on optimal risk and optimal classifier also hold for the distribution perturbing adversary.

Using our characterization of $R^*_\epsilon$ in terms of $D_\epsilon$, we obtain the optimal risk attainable for classification of real-world datasets like CIFAR10, MNIST, Fashion-MNIST and SVHN. Moreover, levaraging our optimal coupling strategy for Gaussian distributions, we also obtain lower bounds on optimal risk for Gaussian mixtures based on these datasets. 
These lower bounds have implications for the limits of data augmentation strategies using Gaussian perturbations. 
Our bounds on adversarial risk are classifier agnostic, and only depend on the data disributions. 
In addition, our bounds are efficiently computable for empirical/mixture distributions via reformulation as a linear program.
However,  our characterization of $R^*_\epsilon$ in terms of $D_\epsilon$ is limited to the binary classification setting. It is not clear if a similar characterization is possible for multi-class classification, perhaps using multi-marginal optimal transport theory \citep{Pas15}. 

Finally, we remark that analzing the $D_\epsilon$ optimal transport cost may be interesting in itself. The optimal transport cost $c_\epsilon(x, x') = \1\{d(x, x')>2\epsilon\}$ is discontinuous and does not satisfy triangle inequality. This makes it hard to analyse $D_\epsilon$ using standard techniques in optimal transport literature. For instance, it would be interesting tighten the bounds from~\citep{Jog20} concerning rates of converges of $D_\epsilon$ between empirical distributions converges to $D_\epsilon$ between the true data-generating distributions.

\section*{Acknowledgements}

We thank the AE and anonymous reviewers for their critical feedback that has led to a much improved manuscript. VJ acknowledges support from NSF grants CCF-1841190, CCF-1907786, and CCF-1942134, and the Nvidia GPU grant program.
 
\bibliographystyle{plain}
\bibliography{ref}

\appendix

\section{Proofs for Section~\ref{sec: opt risk 01 loss}}

\subsection{Strassen's theorem}\label{app: strassen}
Strassen's theorem is a special case for the Kantorovich duality in the case of a $0$-$1$ loss. The statement provided below is as in Villani \citep[Corollary 1.28]{Vil03}:

\begin{lemma}\label{lem_villani}
Let the input $X$ be drawn from a Polish space $\mathcal{X}$. Let $\Pi(p_0, p_1)$ be the set of all probability measures on $\mathcal{X}\times \mathcal{X}$ with marginals $p_0$ and $p_1$. Then for $\epsilon\geq 0$ and $A\subseteq\cX$,
\begin{align*}
\inf_{\pi\in\Pi(p_0, p_1)} \pi[d(x, x' )>\epsilon]
= \sup_{A\ \ closed} \left\{ p_0(A) - p_1(A^{\oplus \epsilon}) \right\}.
\end{align*}
\end{lemma}

\subsection{Proof of Lemma~\ref{lemma: oplus-ominus}}\label{app: lemma: oplus-ominus}

Let $A$ be a closed set and let $B$ be the closed ball of radius $\epsilon$. Fix $\delta > 0$. Let $\{z_i\}_{i \geq 1}$ be a sequence of points in $A^{\oplus \epsilon}$ converging to a limit $z$. Assume without loss of generality that $d(z_i, z) < \delta/2$. We shall show that $z \in A^{\oplus \epsilon}$ as well. Note that every $z_i$ admits an expression $z_i = a_i + b_i$, where $a_i \in A$ and $b_i \in B$. Since $B$ is a compact set, there exists a subsequence among the $\{b_i\}$ sequence that converges to $b^* \in B$. Fix a $\delta >0$ and pick a subsequence $\{\tilde b_i\}_{i \geq 1}$ such that $\tilde b_i \to b^*$ and $|\tilde b_i - b^*| < \delta/2$ for all $i > 0$. Denote the corresponding subsequence of $\{a_i\}$ by $\{\tilde a_i\}$ and $\{z_i\}$ by $\{\tilde z_i\}$. Observe that 
\begin{align*}
z-\tilde a_i = (z-\tilde z_i) + (b^*-\tilde b_i) - b^*,
\end{align*}
and so by the triangle inequality
\begin{align*}
d(z, \tilde a_i) < \delta/2 + \delta/2 + \epsilon = \epsilon+\delta.
\end{align*}
Thus $\tilde a_i \in B(z, \epsilon+\delta) \cap A$, which is a compact set, giving a convergent subsequence within the $\{\tilde a_i\}$ sequence. Let that subsequence converge to $a^*$. We must have $a^*\in A$ and $b^* \in B$ since $A$ and $B$ are closed. This means $z = a^* + b^*$ must lies in $A^{\oplus \epsilon}$, which shows that $A^{\oplus \epsilon}$ is closed.

Recall that $A^{\ominus \epsilon} = ((A^c)^{\oplus \epsilon})^c$. Since $A^c$ is an open set, it is enough to show that $C^{\oplus \epsilon}$ is open if $C$ is open. Let $z \in C^{\oplus \epsilon}$, which means $z = c + b$ for some $c \in C$ and $b \in B$. Consider a small open ball of radius $\delta$ around $c$, called $N_\delta(c)$ that lies entirely in $C$. This is possible since $C$ is assumed to be open. Now observe that $N_\delta(z) \subseteq C^{\oplus \epsilon}$, since $N_\delta(z) = N_\delta(c) + b$. This shows that every point $z \in C^{\oplus \epsilon}$ admits a small ball around it that is contained in $C^{\oplus \epsilon}$, or equivalently, $C^{\oplus \epsilon}$ is open. This completes the proof.

\subsection{Proof of Lemma~\ref{lemma: equivalence}}\label{app: lemma: equivalence}

Let $x \in A^{\oplus \epsilon}$. Then there exists an $a \in A$ such that $d(x, a) \leq \epsilon$, which means $d(x, A) \leq \epsilon$, and so $x \in A^{\epsilon}$. This shows that $A^{\oplus \epsilon} \subseteq A^\epsilon$. 

To prove the reverse direction, suppose $x \in A^\epsilon$. This means we can a sequence of points $\{a_i\}$ such that $a_i \in A$ and $\liminf_i d(x, a_i) \leq \epsilon$. Fix a $\delta >0$ and assume without loss of generality that $d(x, a_i) \leq \epsilon+\delta$ for all $i > 0$. Then  $a_i \in B(x, \epsilon+\delta) \cap A$ for all $i > 0$. As $A$ is closed, the set $B(x_i, \epsilon+\delta) \cap A$ is compact, and there exists a subsequence $\{\tilde a_i\}$ that converges to $a^* \in A$. By the triangle inequality, $d(x, a^*) \leq d(x, \tilde a_i) + d(\tilde a_i, a^*)$. Taking $\liminf_i$ on both sides yields 
\begin{align*}
d(x, a^*) \leq \liminf_i d(x, \tilde a_i) \leq \epsilon.
\end{align*}
This implies $x \in A^{\oplus \epsilon}$, and we conclude $A^\epsilon \subseteq A^{\oplus \epsilon}$.

\subsection{Proof of Lemma~\ref{lemma: thick and thin fixed}}\label{app: lemma: thick and thin}
We claim that a point $x \in A^{ \ominus \epsilon}$ if and only if $B(x, \epsilon)$ lies entirely in $A$. If this were not the case, then we could find a $y \in A^c$ such that $d(x,y) \leq \epsilon$, and so $x \in (A^c)^{\oplus \epsilon}$, which implies $x \not\in ((A^c)^{\oplus \epsilon})^c = A^{ \ominus \epsilon}$. Conversely, if $B(x, \epsilon) \in A$ then $d(x, y) > \epsilon$ for all $y \in A^c$, and so $x \notin (A^c)^{\oplus \epsilon}$, which means $x \in ((A^c)^{\oplus \epsilon})^c = A^{\ominus \epsilon}$. This observation implies that $(A^{\ominus \epsilon})^{\oplus \epsilon} \subseteq A$.

Using the above logic for $A^{\oplus \epsilon}$, we see that a point $x \in (A^{\oplus \epsilon})^{\ominus \epsilon}$ if and only if $B(x, \epsilon) \subseteq A^{\oplus \epsilon}$. By definition of $A^{\oplus \epsilon}$, every point $x \in A$ satisfies $B(x, \epsilon) \subseteq A^{\oplus \epsilon}$. Thus, if $x \in A$ then $x \in (A^{\oplus \epsilon})^{\ominus \epsilon}$. Equivalently, $A \subseteq (A^{\oplus \epsilon})^{\ominus \epsilon}$.

\section{Proofs for Section~\ref{sec: couplings}}

\subsection{Proof of Theorem~\ref{thm: uniform coupling}}\label{app: uniform coupling proof}

Like in the proof for Theorem~\ref{thm: same mean}, we prove Theorem~\ref{thm: uniform coupling} by partitioning the real line into several regions for $\mu$ and $\nu$, and transporting mass between these regions. Figure~\ref{fig: coupling_uniform} shows the optimal coupling for the case when $I^{2\epsilon}\subseteq J$.

\begin{figure}[t]
\begin{center}
\includegraphics[scale = 0.5]{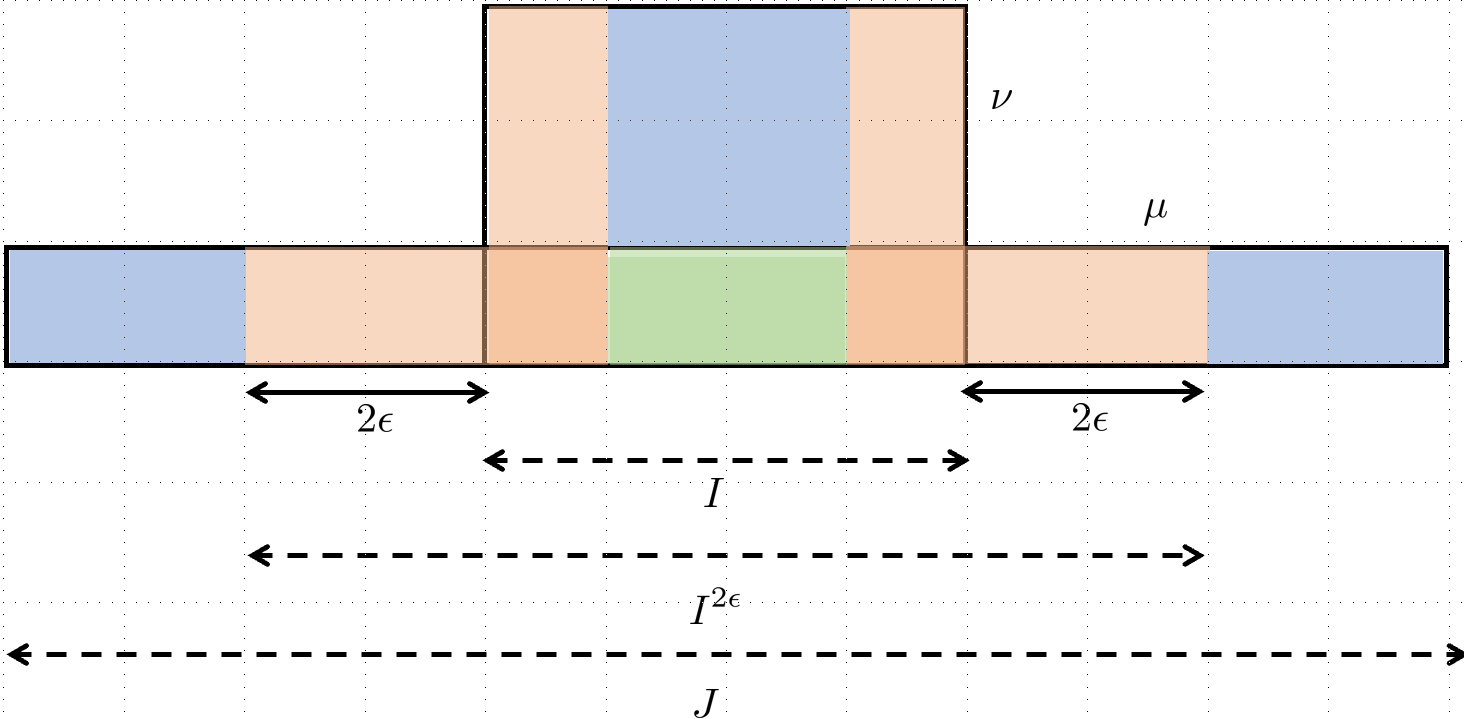}
\end{center}
\caption{Optimal coupling for two uniform distributions. The region shaded in green is kept in place (at no cost).  The two regions shaded in orange are transported monotonically from either side at a cost not exceeding $2\epsilon$ per unit mass. The remaining region in blue is moved at the cost of 1 per unit mass.  }
\label{fig: coupling_uniform}
\end{figure}

We first prove a lower bound. Choose the set $A = I$, we have that
\begin{align}\label{eq: uniform dist LB}
D_\epsilon(\mu, \nu) \geq \mu(A) - \nu(A^{2\epsilon}) = 1 - \nu(I^{2\epsilon}).
\end{align}
To establish the upper bound, we need to find a coupling that transports $\mu$ to $\nu$ such that the cost of transportation is bounded above by $1-\nu(I^{2\epsilon})$.  Without loss of generality, let $I = [-w_1,w_1]$ and $J = [c-w_2, c+w_2]$ for some $c>0$ and $0< w_1\leq w_2$.

\paragraph{Case 1: $2\epsilon< w_2 - w_1$.} We split the analysis into the following five sub-cases.

\paragraph{Case 1(a): $c\in [w_1+w_2+2\epsilon, \infty)$.} In this case, the intervals $I$ and $J$ are separated by at least $2\epsilon$. 
Hence, $\nu(I^{2\epsilon}) = 0$, and therefore, $D_\epsilon(\mu, \nu) \geq 1 - \nu(I^{2\epsilon}) = 1$. Combining this with the fact that  $D_\epsilon(\mu, \nu) \leq 1$, we get that $D_\epsilon(\mu, \nu) = 1 = 1 - \nu(I^{2\epsilon})$.

\paragraph{Case 1(b): $c\in [-w_1+w_2-2\epsilon,  w_1+w_2+2\epsilon)$.} In this case, $\nu(I^{2\epsilon}) = \nu([c-w_2, w_1+2\epsilon]) = (w_1+2\epsilon-c+w_2)/(2w_2) \leq 1$.
Since $\mu([-w_1, w_1]) = 1 \geq  \nu([c-w_2, w_1+2\epsilon])$, there must exist a $u\in [-w_1, w_1]$ such that $\mu([u, w_1]) = \nu([c-w_2, w_1+2\epsilon])$.
Solving for $u$, we get the following.
\begin{align*}
&\frac{w_1-u}{2w_1} = \mu([u, w_1]) = \nu([c-w_2, w_1+2\epsilon]) = \frac{w_1+2\epsilon-c+w_2}{2w_2}\\
&\implies u = w_1 - \frac{w_1}{w_2}(w_1+2\epsilon-c+w_2).
\end{align*}
Since $w_1/w_2<1$, the above equation for $u$ shows that $u> w_1 - (w_1+2\epsilon-c+w_2) = c-w_2-2\epsilon$. Hence, $(c-w_2)-u<2\epsilon$.

Let $\mu_0$ be the restriction of $\mu$ to $[u, w_1]$ and $\nu_0$ be the restriction of $\nu$ to $[c-w_2, w_1+2\epsilon]$. Then, by construction, $\mu_0(\real) = \nu_0(\real)$. By Lemma~\ref{lem: push_forward}, we have a monotone transport map $T: [u, w_1]\to [c-w_2, w_1+2\epsilon]$ that transports $\mu_0$ to $\nu_0$ given by $T(x) = \frac{w_1+2\epsilon-c+w_2}{w_1-u}(x-u)+(c-w_2)$. Note that $T$ transports $u$ to $c-w_2$ and $w_1$ to $w_1+2\epsilon$. Also, $T(x)>x$. Since $T$ has a slope greater than $1$, $T(x)-x$ is an increasing function. Moreover, $T(w_1)-w_1 = 2\epsilon$ and $T(u)-u = (c-w_2)-u <2\epsilon$. Hence, $|T(x)-x|\leq 2\epsilon$ for all $x\in [u, w_1]$. Hence, $D_\epsilon(\mu_0, \nu_0)=0$. Therefore, $D_\epsilon(\mu, \nu) \leq 1- \min(\mu_0, \nu_0) = 1- \nu([c, w_1+2\epsilon]) = 1 - \nu(I^{2\epsilon})$. Combining with the lower bound in \eqref{eq: uniform dist LB}, we conclude that $D_\epsilon(\mu, \nu) = 1 - \nu(I^{2\epsilon})$.

\paragraph{Case 1(c): $c\in (-w_2+w_1+2\epsilon,  -w_1+w_2-2\epsilon)$.} In this case, $\nu(I^{2\epsilon}) = \nu([-w_1-2\epsilon, w_1+2\epsilon]) = (2w_1+4\epsilon)/(2w_2) \leq 1$.
Since $\mu([0,w_1]) = 1/2 > \nu(0, w_1+2\epsilon)$, there must exists a $v\in [0,w_1]$ such that $\mu([v,w_1]) =  \nu([0, w_1+2\epsilon])$.
Let $\mu_+$ be the restriction of $\mu$ to $[u, w_1]$ and $\nu_+$ be the restriction of $\nu$ to $[0, w_1+2\epsilon]$.
Then, by construction, $\mu_+(\real) = \nu_+(\real)$.
Similar to the map $T$ in case 1b, there exists a monotone transport map $T_+: [u, w_1]\to [0, w_1+2\epsilon]$ such that $|T_+(x)-x|\leq 2\epsilon$. Hence, $D_\epsilon(\mu_+, \nu_+)=0$. Similarly, let $\mu_-$ be the restriction of $\mu$ to $[-w_1, -u]$ and $\nu_+$ be the restriction of $\nu$ to $[-w_1-2\epsilon, 0]$. Then by symmetry, there also exists a monotone transport map $T_-: [-w_1, -u]\to [-w_1-2\epsilon, 0]$ such that $|T_-(x)-x|\leq 2\epsilon$. Hence,  $D_\epsilon(\mu_-, \nu_-)=0$. Therefore, 
\begin{align*}
D_\epsilon(\mu, \nu) &\leq 1- [\min(\mu_+, \nu_+)+ \min(\mu_-, \nu_-)]\\
&= 1 - [\nu([0, w_1+2\epsilon]) + \nu([-w_1-2\epsilon, 0])]\\
&= 1 - \nu([-w_1-2\epsilon, w_1+2\epsilon])\\
&= 1 - \nu(I^{2\epsilon}).
\end{align*}
Combining with the lower bound in \eqref{eq: uniform dist LB}, we conclude that $D_\epsilon(\mu, \nu) = 1 - \nu(I^{2\epsilon})$.

\paragraph{Case 1(d): $c\in (-w_1-w_2-2\epsilon, w_1-w_2+2\epsilon]$.} The geometry of this case is a mirror image of that in case 1b. Hence, just as in case 2, we have $D_\epsilon(\mu, \nu) = 1 - \nu(I^{2\epsilon})$.

\paragraph{Case 1(e): $c\in (-\infty, -w_1-w_2-2\epsilon]$.} Like in case 1, the intervals $I$ and $J$ are separated by at least $2\epsilon$. Hence, similar to case 1, we get that $D_\epsilon(\mu, \nu) = 1 = 1 - \nu(I^{2\epsilon})$.

\paragraph{Case 2: $2\epsilon\geq  w_2 - w_1$.} In this case, we have the following sub-cases.

\paragraph{Case 2(a): $c\in [w_1+w_2+2\epsilon, \infty)$.} Like in case 1a, the intervals $I$ and $J$ are separated by at least $2\epsilon$. 
Hence, $D_\epsilon(\mu, \nu) = 1 = 1 - \nu(I^{2\epsilon})$.

\paragraph{Case 2(b): $c\in [w_1-w_2+2\epsilon,  w_1+w_2+2\epsilon)$.} Since $[w_1-w_2+2\epsilon,  w_1+w_2+2\epsilon) \subseteq [-w_1+w_2-2\epsilon,  w_1+w_2+2\epsilon)$, the coupling obtained in Case 1b can be directly applied in this case. Hence, we again have 
$D_\epsilon(\mu, \nu) = 1 = 1 - \nu(I^{2\epsilon})$.

\paragraph{Case 2(c): $c\in (-w_1+w_2-2\epsilon, w_1-w_2+2\epsilon)$.} In this case, the supports of $\mu$ and $\nu$ are within $2\epsilon$ of each other. More specifically, $J\subseteq I^{2\epsilon}$. Hence, $\nu(I^{2\epsilon})=1$.
Let $T$ denote the monotone transport map from  $\mu$ and $\nu$ as defined in Lemma~\ref{lem: push_forward}. Then, $T(x) = \frac{w_2}{w_1}(x-w_1) + (c+w_2)$. Note that $T$ maps $[-w_1, w_1]$ to $[c-w_2, c+w_2]$ monotonically. Since the supports of $\mu$ and $\nu$ are within $2\epsilon$ of each other, we have $|T(x)-x|\leq 2\epsilon$. Hence, $D_\epsilon(\mu, \nu) = 0 = 1 - \nu(I^{2\epsilon})$.

\paragraph{Case 2(d): $c\in ( -w_1-w_2-2\epsilon, -w_1+w_2-2\epsilon]$.} This case is a mirror image of case 2b and hence the result 
$D_\epsilon(\mu, \nu) =1 - \nu(I^{2\epsilon})$ remains the same.

\paragraph{Case 2(e): $c\in (-\infty, -w_1-w_2-2\epsilon]$.} Like in case 1a, the intervals $I$ and $J$ are separated by at least $2\epsilon$. 
Hence, $D_\epsilon(\mu, \nu) = 1 = 1 - \nu(I^{2\epsilon})$.

It is easily checked that the error attained by the proposed classifier also matches the bound, which completes the proof.

\subsection{Proof of Theorem~\ref{thm: triangle coupling}}\label{app: triangle coupling proof}

We have the following cases:

	\paragraph{ Case 1: $m_2-m_1 >  \delta_1+\delta_2+2\epsilon$.}
	
In this case $\mu$ and $\nu$ have disjoint supports separated by at least $2\epsilon$. Moreover, $\mu(A^{\ominus \epsilon})=1$ and $\nu(A^{\oplus \epsilon}) = 0$. Then,
\begin{align*}
	D_\epsilon(\mu, \nu) = \sup_{A\ closed} \mu(A^{\ominus \epsilon}) - \nu(A^{\oplus \epsilon}) \geq \mu(A^{\ominus \epsilon}) - \nu(A^{\oplus \epsilon}) = 1.
\end{align*}
Combining the above inequality with the fact that $D_\epsilon(\mu, \nu)\leq 1$, we get $D_\epsilon(\mu, \nu) = 1$.

	\paragraph{ Case 2: $m_2-m_1 <  \delta_2-\delta_1-2\epsilon$.}
	
	In this case,
	\begin{align*}
	|(m_2+\epsilon) - (m_1-\epsilon)| = |(m_2-m_1) + 2\epsilon| = (m_2-m_1) + 2\epsilon < \delta_2-\delta_1,\\
	|(m_2-\epsilon) - (m_1+\epsilon)| = |(m_2-m_1) - 2\epsilon| \leq |m_2-m_1| + 2\epsilon < \delta_2-\delta_1.
	\end{align*}
	Hence, by Lemma~\ref{lemma: triangular intersection}, the equations $f(x+\epsilon) = g(x-\epsilon)$ and $f(x-\epsilon) = g(x+\epsilon)$ have exactly two solutions each, on the supports of $\Delta(m_1- \epsilon, \delta_1)$ and $\Delta(m_1+ \epsilon, \delta_1)$ respectively. Hence, $l$ must be the minimum of the two solutions to $f(x+\epsilon) = g(x-\epsilon)$ and $r$ must be the maximum of the two solutions to $f(x-\epsilon) = g(x+\epsilon)$. As in the proof of Theorem~\ref{thm: diff mean}, we divide the real line into five regions as shown in Table~\ref{table: triangular_partition_five}, where $l'$ is the leftmost point such that $\mu([l+\epsilon, l']) = \nu([l-\epsilon, l'])$ and $r'$ is the rightmost point such that $\mu([r', r-\epsilon]) = \nu([r', r+\epsilon])$. 
	Observe that by construction, $f(x)\leq g(x+2\epsilon)$ for $x\in [r-\epsilon, m_1+\delta_1]$. Hence by Lemma~\ref{lemma: shift}, $D_\epsilon(\mu_{++}, \nu_{++})=0$. Similarly, we also get $D_\epsilon(\mu_{--}, \nu_{--})=0$.
	
	We will now use Lemma~\ref{lemma: scrunch} to show that $D_\epsilon(\mu_{-}, \nu_{-})=0$.
	Let $a = l-\epsilon$, $a' = l+\epsilon$, $b = l'$ and $\tilde{t} = l'-2\epsilon$. Let $t$ be the first coordinate of the intersection point of two line segments, one joining $(a, g(a))$ and $(b, g(b))$, and the other joining $(a', f(a'))$ and $(b, f(b))$.
	 The following three conditions are satisfied by $\mu_{-}$ and $\nu_{-}$. (1) The support of $\nu_{-}$ is $[a,b]$ and the support of $\mu_{-}$ is $[a',b] = [a+2\epsilon, b]$. (2) $g(x)\geq f(x)$ for $x\in [a, t)$ and $f(x)\geq g(x)$ for $x\in (t, b]$. (3) $g(x)\leq f(x+2\epsilon)$ for $x\in [a, \tilde{t})$ and the the interval $(\tilde{t}, b-2\epsilon]$ is empty because $\tilde{t}= b-2\epsilon$.
	 Hence, $D_\epsilon(\mu_{-}, \nu_{-})=0$. Similarly, $D_\epsilon(\mu_{+}, \nu_{+})=0$.
	 
	 Finally, $D_\epsilon(\mu_0, \nu_0)=0$. This is because $f(x)\geq g(x)$ for $x\in [l', r']$ where $[l', r']$ is the support of both $\mu_0$ and $\nu_0$ and so an identity map $T(x)=x$ may be used to transport all the mass from $\nu_0$ to $\mu_0$ at zero cost.
	 
	 Like in the proof of Theorem~\ref{thm: same variance}, we can upper bound $D_\epsilon(\mu, \nu)$ as follows.
	 \begin{align*}
	 D_\epsilon(\mu, \nu) &\leq 1 - (\nu([l-\epsilon, r+\epsilon]) + \mu([m_1-\delta_1, l+\epsilon]) + \mu([r-\epsilon, m_1+\delta_1])\\
	 &= \mu(A^{\ominus \epsilon}) -\nu(A^{\oplus \epsilon}).
	 \end{align*}
	 Since $D_\epsilon(\mu, \nu) = \sup_{B\ closed} \mu(B^{\ominus \epsilon}) - \nu(B^{\oplus \epsilon})$, the above inequality turns to an equality.
	
	\paragraph{ Case 3: $\delta_2-\delta_1-2\epsilon<m_2-m_1 <  \delta_2+\delta_1+2\epsilon$. }
	
	In this case,
	\begin{align*}
	(m_2-\epsilon) - (m_1+\epsilon) = (m_2-m_1) - 2\epsilon < \delta_2+\delta_1,\\
	(m_1+\epsilon) - (m_2-\epsilon) = 2\epsilon - (m_2-m_1) < \delta_2+\delta_1.
	\end{align*}	
	Hence, $|(m_2-\epsilon) - (m_1+\epsilon)|< \delta_2+\delta_1$. By Lemma~\ref{lemma: triangular intersection}, the equation $f(x-\epsilon) = g(x+\epsilon)$ has either one or two solutions. Therefore, $r$ must be the rightmost solution to $f(x-\epsilon) = g(x+\epsilon)$.
	
	We will split the analysis into three sub-cases.

		\paragraph{ Case 3(a): $m_2-m_1 > \delta_2 - \delta_1+2\epsilon$.}

		We will decompose $\mu$ and $\nu$ into two mutually singular positive measures each.
		Let $\mu_-$ and $\mu_+$ be the restriction of $\mu$ to the intervals $[m_1-\delta_1, r-\epsilon]$ and $[r-\epsilon, m_1+\delta_1]$ respectively. Let $\nu_-$ and $\nu_+$ be the restriction of $\nu$ to the intervals $[m_2-\delta_2, r+\epsilon]$ and $[r+\epsilon, m_2+\delta_2]$ respectively. 
		The following inequality shows that the support of $\nu_-$ is of a lesser length than that of $\mu_-$.
		\begin{align*}
			[(r+\epsilon)-(m_2-\delta_2)] - [(r-\epsilon)-(m_1-\delta_1)] = \delta_2 - \delta_1 + 2\epsilon - (m_2-m_1)<0.
		\end{align*}		
		It follows that the support of $\nu_+$ is of a greater length than that of 	$\mu_+$.
		By construction, $g(x-2\epsilon)\leq f(x)$ for $x\in [m_2-\delta_2, r+\epsilon]$. Hence, by Lemma~\ref{lemma: shift}, $D_\epsilon(\mu_-, \nu_-) = 0$. A similar analysis shows that $D_\epsilon(\mu_+, \nu_+) = 0$. Hence,
		\begin{align*}
		D_\epsilon(\mu, \nu) &\leq 1 - \min(\mu_-(\real), \nu_-(\real)) - \min(\mu_+(\real), \nu_+(\real))\\
		&= 1 - \mu([r-\epsilon, \infty)) - \nu([r+\epsilon, \infty))\\
		&= \mu(A^{\ominus \epsilon}) -\nu(A^{\oplus \epsilon}).
		\end{align*}
		Since $D_\epsilon(\mu, \nu) = \sup_{B\ closed} \mu(B^{-\epsilon}) - \nu(B^\epsilon)$, the above inequality turns to an equality.

		\paragraph{ Case 3(b):  $ \delta_2 - \delta_1 <  m_2-m_1 \leq \delta_2 - \delta_1 +2\epsilon$.}

		Let $\mu_-, \mu_+, \nu_-$ and $\nu_+$ be as defined in case 3(a). 
		The following inequality shows that the support of $\mu_+$ is smaller than that of $\nu_+$.
		\begin{align*}
			[(m_2+\delta_2) - (r+\epsilon)] - [(m_1+\delta_1) - (r-\epsilon)]
			= (m_2-m_1) + \delta_2-\delta_1 -2\epsilon >0.
		\end{align*}			
		Moreover, $f(x)\leq g(x+2\epsilon)$ for $x\in [r-\epsilon, m_1+\delta_1]$.
		Hence by Lemma~\ref{lemma: shift}, $D_\epsilon(\mu_+, \nu_+) = 0$.
		
		We will now show that $D_\epsilon(\mu_{-}, \nu_{-})=0$ by verifying the conditions of Lemma~\ref{lemma: scrunch}.
		Since $2\epsilon < 2\delta_1$, we have the following.
		\begin{align*}
			\delta_2 - \delta_1 <  m_2-m_1 \leq \delta_2 - \delta_1 +2\epsilon < \delta_2 - \delta_1 + 2\delta_1 = \delta_2 + \delta_1.
		\end{align*}				
		Hence, by Lemma~\ref{lemma: triangular intersection}, there is exactly one point of intersection of $f(x)$ and $g(x)$ on the support of $\mu$. Let $t$ be the first coordinate of that point. Let $a = m_2-\delta_2-2\epsilon$, $a' = a+2\epsilon$ and $b= r+\epsilon$. Then, (1) the support of $\mu_-$ is $[m_1-\delta_1, r-\epsilon]$ which is a subset of $[a, b]$, and the support of $\nu_-$ is $[a', b]$. (2) $f(x)\geq g(x)$ for $x\in (a, t]$ and $f(x)\leq g(x)$ for $x\in (t, b]$. Hence, the first two conditions of Lemma~\ref{lemma: scrunch} are verified.
		To verify, the third condition, we note the following.
		\begin{align*}
			(m_2-2\epsilon)-m_1 = m_2-m_1 - 2\epsilon < \delta_2-\delta_1,\\
			m_1-(m_2-2\epsilon) = m_1-m_2+2\epsilon < 2\epsilon < \delta_2-\delta_1.
		\end{align*}
		Hence, by  Lemma~\ref{lemma: triangular intersection}, $f(x)-g(x+2\epsilon)=0$ exactly twice on the support of $\mu$. The greater of the two will be $r-\epsilon$. Let $\tilde{t}$ be the lesser of the two. Then, $\tilde{t}<r-\epsilon = b-2\epsilon$. Further, $f(x)\leq g(x+2\epsilon)$ for $x\in [a, \tilde{t})$ and $f(x)\geq g(x+2\epsilon)$ for $x\in (\tilde{t}, b-2\epsilon]$.
	 Hence, $D_\epsilon(\mu_{-}, \nu_{-})=0$ by Lemma~\ref{lemma: scrunch}. 
	 Therefore, the optimal risk and optimal classifier remain the same as in case 3(a).

		\paragraph{ Case 3(c): $m_2-m_1 \leq \delta_2 - \delta_1$.}
		
		We will partition the real line into four regions as shown in Table~\ref{table: triangular_partition_four}, where $l'$ is the leftmost point such that $\mu([m_1-\delta_1, l']) = \nu([m_2-\delta_2, l'])$ and $r'$ is as defined in case 2. Since $\mu_+, \nu_+, \mu_{++}$ and $\nu_{++}$ are defined in an identical manner to case 2, we get $D_\epsilon(\mu_{+}, \nu_{+}) = D_\epsilon(\mu_{++}, \nu_{++}) = 0$.
		
		We will now show $D_\epsilon(\mu_{--}, \nu_{--})=0$ using Lemma~\ref{lemma: scrunch}. Let $a=m_1-\delta_1-2\epsilon$, $a'=a+2\epsilon$, $b = l'$ and $\tilde{t} = b-2\epsilon$. Since $m_2-m_1 \leq \delta_2 - \delta_1$, by Lemma~\ref{lemma: triangular intersection}, $f(x)-g(x)=0$ has exactly two solutions. Let $t$ be the lesser of the two. Then, (1) the support of $\nu_{--}$ is $[m_2-\delta_2, b]$ which is a subset of $[a,b]$ and the support of $\mu_{--}$ is $[a',b]$. (2) $g(x)\geq f(x)$ for $x\in [a, t)$ and $f(x)\geq g(x)$ for $x\in (t, b]$. (3) $g(x)\leq f(x+2\epsilon)$ for $x\in [a, \tilde{t})$ and the the interval $(\tilde{t}, b-2\epsilon]$ is empty because $\tilde{t}= b-2\epsilon$. Hence, $D_\epsilon(\mu_-, \nu_-)=0$.
		
		Finally, $D_\epsilon(\mu_-, \nu_-)=0$ because $f(x)\geq g(x)$ for $x\in [l', r']$ and the identity map $T(x)=x$ transports all the mass from $\nu_-$ to $\mu_-$ at zero cost.
		
		Overall, we have the following inequality.
		\begin{align*}
		D_\epsilon(\mu, \nu)
		&\leq 1- (\nu([m_2-\delta_2, l']) + \nu([l', r']) + \nu([r', r+\epsilon]) + \mu([r-\epsilon, m_1+\delta_1]))\\
		&= \mu(A^{\ominus \epsilon}) - \nu(A^{\oplus \epsilon}).
		\end{align*}
		As in Case 2, we conclude that $D_\epsilon(\mu, \nu) = \mu(A^{\ominus \epsilon}) - \nu(A^{\oplus \epsilon})$.

\begin{figure}[t]
    \centering
    \includegraphics[scale=0.5]{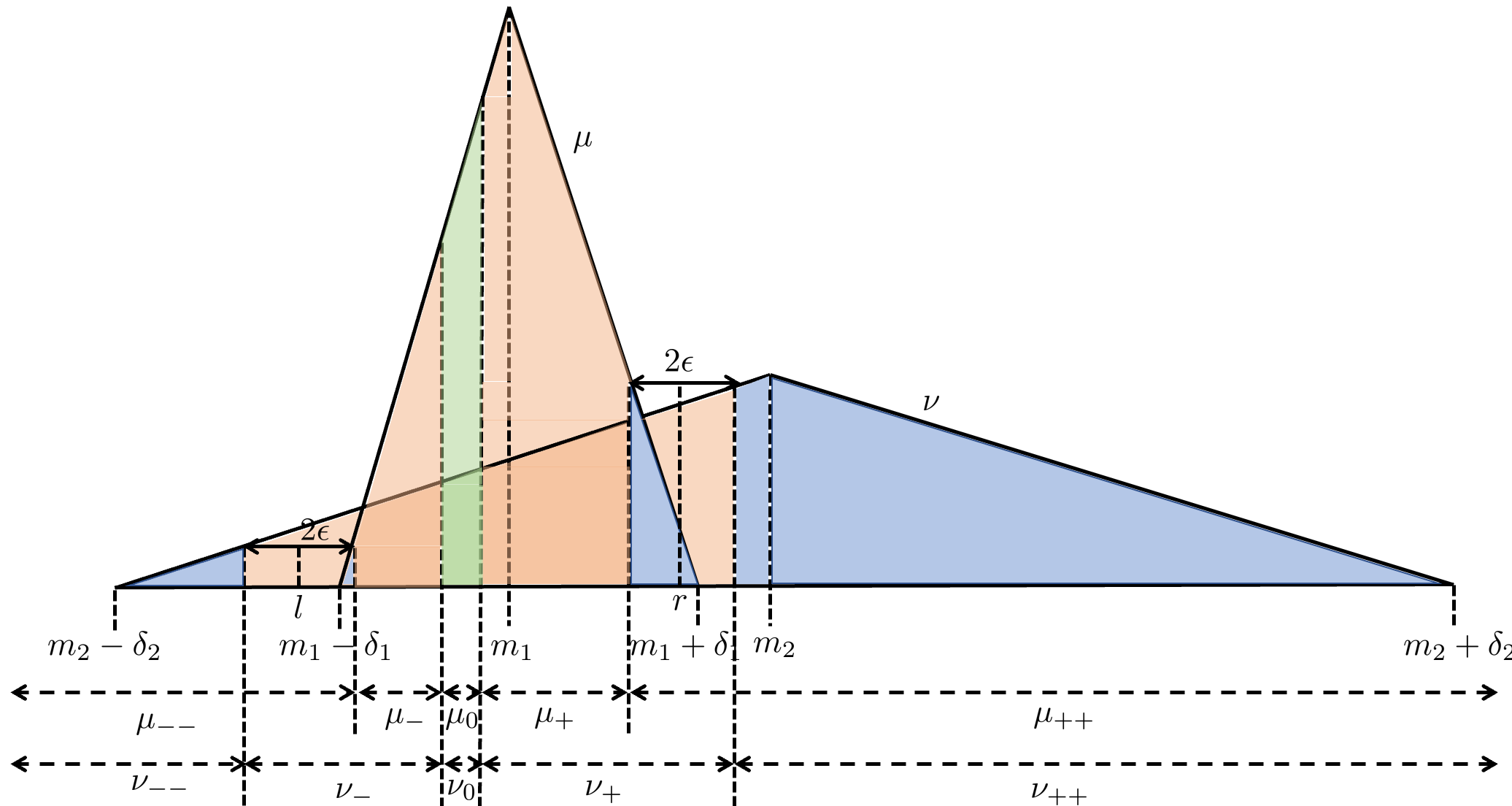}
    \caption{Optimal transport coupling for triangular distributions $\mu$ and $\nu$. As in the proof of Theorem~\ref{thm: diff mean}, we divide the real line into five regions. The transport plan from $\mu$ to $\nu$ consists of five maps transporting $\mu_{--} \to \nu_{--}$ (blue regions to the left), $\mu_- \to \nu_-$ (orange regions to the left), $\mu_0 \to \nu_0$ (green regions in the middle), $\mu_+ \to \nu_+$ (orange regions to the right), and $\mu_{++} \to \nu_{++}$ (blue regions to the right).}
    \label{fig: triangular}
\end{figure}

\begin{table}[!htb]
\begin{center}
\begin{tabular}{|c|c|}
\hline
$\mu_{--}$ &$(m_1-\delta_1, l+\epsilon]$  \\ \hline
$\mu_{-}$ &$(l+\epsilon, l']$  \\ \hline
$\mu_{0}$ &$(l', r')$  \\ \hline
$\mu_{+}$ &$[r', r-\epsilon)$  \\ \hline
$\mu_{++}$ &$[r-\epsilon, m_1+\delta_1)$  \\ \hline
\end{tabular}
%\end{table}
\quad
%\begin{table}[]
\begin{tabular}{|c|c|}
\hline
$\nu_{--}$ &$(m_2-\delta_2, l-\epsilon]$  \\ \hline
$\nu_{-}$ &$(l-\epsilon, l']$  \\ \hline
$\nu_{0}$ &$(l', r')$  \\ \hline
$\nu_{+}$ &$[r', r+\epsilon)$  \\ \hline
$\nu_{++}$ &$[r+\epsilon, m_2+\delta_2)$  \\ \hline
\end{tabular}
\end{center}
\caption{The real line is partitioned into five regions for $\mu$ and $\nu$ for Case 2.}
\label{table: triangular_partition_five}
\end{table}

\begin{table}[!htb]
\begin{center}
\begin{tabular}{|c|c|}
\hline
$\mu_{--}$ &$(m_1-\delta_1, l']$  \\ \hline
$\mu_{-}$ &$(l', r')$  \\ \hline
$\mu_{+}$ &$[r', r-\epsilon)$  \\ \hline
$\mu_{++}$ &$[r-\epsilon, m_1+\delta_1)$  \\ \hline
\end{tabular}
%\end{table}
\quad
%\begin{table}[]
\begin{tabular}{|c|c|}
\hline
$\nu_{--}$ &$(m_2-\delta_2, l']$  \\ \hline
$\nu_{-}$ &$(l', r')$  \\ \hline
$\nu_{+}$ &$[r', r+\epsilon)$  \\ \hline
$\nu_{++}$ &$[r+\epsilon, m_2+\delta_2)$  \\ \hline
\end{tabular}
\end{center}
\caption{The real line is partitioned into four regions for $\mu$ and $\nu$ for Case 3(c).}
\label{table: triangular_partition_four}
\end{table}

\end{document}